%% file: main.tex
\documentclass[10pt,twocolumn,letterpaper]{article}

 \usepackage[pagenumbers]{cvpr} 

\input{preamble}
\definecolor{cvprblue}{rgb}{0.21,0.49,0.74}
\usepackage[pagebackref,breaklinks,colorlinks,allcolors=cvprblue]{hyperref}
\usepackage{amsmath} 
\usepackage{wrapfig}
\usepackage{graphicx} 
\usepackage{amsthm}
\usepackage[table]{xcolor}
\usepackage{multirow}
\usepackage{booktabs}
\usepackage{booktabs}
\usepackage{multirow}
\usepackage[ruled,vlined]{algorithm2e}
\usepackage{subcaption}
\usepackage{placeins}
\usepackage{booktabs}
\usepackage{tabularx}
\usepackage{array}
\usepackage{marvosym}

\def\confName{CVPR}
\def\confYear{2027}

\newtheorem{proposition}{Proposition}

\definecolor{Better}{rgb}{0.18, 0.407, 0.266}
\definecolor{Worse}{rgb}{0.35, 0.35, 0.35}

\newcommand{\ours}{\textbf{PRiSM}}
\newcommand{\imp}[1]{$_{{\textbf{\textcolor{Better}{#1}}}}$}
\newcommand{\wor}[1]
{$_{{\textbf{\textcolor{Worse}{#1}}}}$}
\newcommand{\gc}{\cellcolor{gray!15}}
\newcommand{\dashfill}{%
  \leavevmode
  \cleaders\hbox{\rule[0.5ex]{3pt}{0.3pt}\kern2pt}%
  \hskip0pt plus 1filll\relax
  \kern0pt
}
\newcommand{\dashseparator}{%
  \raisebox{-0.5ex}{%
    \vbox to 2.5ex{%
      \cleaders\vbox{%
        \hrule width 0.4pt height 1.5pt
        \kern 1.5pt
      }\vfill
    }%
  }%
}

\title{PRiSM: Prototype Regularization for Few-Shot VLMs}

\author{
Ghassen Baklouti\textsuperscript{\Letter}
\qquad
Omprakash Chakraborty
\qquad
Jose Dolz
\qquad
Ismail Ben Ayed
\\[0.8em]
ÉTS Montréal, Canada
\\
{\small \Letter\ \texttt{ghassen.baklouti.1@ens.etsmtl.ca}}
}

\begin{document}
\maketitle
\input{sec/0_abstract}    
\input{sec/1_intro}
\input{sec/2_benchmark_description}
\input{sec/3_methodology}

\input{sec/4_experiments}
\input{sec/5-conclusion}

{
    \small
    \bibliographystyle{ieeenat_fullname}
    \bibliography{main}
}

\input{sec/X_suppl}

\end{document}

%% file: sec/0_abstract.tex
\begin{abstract}
Few-shot adaptation methods have gained significant attention recently in the context of Vision-language Models (VLMs). Yet, current benchmarks rely on strong assumptions about the statistics of the adaptation data, e.g., class balance. We question these simplifying assumptions and introduce a more realistic benchmark that varies both the levels of class balance and the effective number of classes in few-shot tasks via Dirichlet sampling. Surprisingly, under our setting, we observe substantial drops in the performances of state-of-the-art methods, more so when the number of labeled samples increases. To mitigate this, we introduce \ours{}, a class-prototype regularization that can be deployed as a plug and play module on top of any existing baseline method, significantly improving performances. Our method optimizes a novel multi-term loss, which includes a regularizer maximizing inter-class pairwise distances, along with additional terms promoting support-feature alignment and fidelity to the baseline prototypes. Furthermore, we introduce an effective and computationally efficient block Majorize-Minimize optimizer for our objective. More specifically, we derive a valid blockwise Lipschitz constant  (i.e., a bound on the Hessian’s spectral norm), which can be computed efficiently via the Gershgorin circle theorem.
Extensive experiments show that \ours{} improves several baselines, with large gains when dealing with severe class imbalance and high numbers of classes.
\end{abstract}

%% file: sec/1_intro.tex
\section{Introduction}
\input{plots/geometry_teas_fig}
Vision–language models (VLMs) ~\cite{radford2021clip, jia2021scaling, xu2023metaclip} have emerged as powerful general-purpose learners, enabling zero-shot recognition by aligning visual and textual representations in a shared embedding space. A common paradigm for downstream adaptation is few-shot learning, where a small set of labeled samples, referred to as the {\em support set}, is used to update the models, which are then evaluated on an unseen {\em query set}. Existing few-shot adaptation methods for VLMs can be broadly grouped into training-free, prompt-learning, and parameter-efficient adaptation approaches.   Training-free methods leverage the support-set statistics without gradient-based optimization. Representative examples include TIP-Adapter ~\cite{zhang2022tipadapter}, which builds a cache model from the support features; APE ~\cite{zhu2023not}, which incorporates a feature-selection step into the cache model; GDA ~\cite{wang2024baseline}, which models each class as a Gaussian distribution with shared covariance in the embedding space; and ProKeR ~\cite{bendou2025proker}, which performs kernel-based prototype refinement. Prompt-learning methods replace the hand-crafted templates with learnable text tokens. The popular CoOp approach ~\cite{zhou2022coop} optimizes static task-specific context vectors, whereas CoCoOp ~\cite{zhou2022cocoop} extends this approach to image-conditioned vectors, improving generalization to unseen classes. Parameter-efficient methods optimize lightweight classifiers or selectively update a small subset of model parameters. LP++ ~\cite{huang2024lp++} combines trainable classifier weights with textual prototypes through learnable class-wise coefficients; CLAP ~\cite{clap24} employs class-adaptive regularization to preserve zero-shot knowledge during linear probing; and ProLIP ~\cite{fahes2026clip} fine-tunes the visual projection layer while penalizing its deviation from pretrained weights.

While the reported empirical performances are strong, we argue that the benchmarks and experimental evaluations in the literature have two important limitations:
\begin{itemize}

\item First, the standard few-shot evaluation protocol~\cite{bendou2025proker, zanella2024boosting,  zhou2022cocoop} relies on a critical yet often overlooked assumption: \emph{class-balanced tasks}. In particular, both the support and query sets are constructed such that the classes are equally represented, implicitly enforcing a {\em uniform prior}. While this protocol simplifies adaptation, it poorly reflects real-world scenarios, where class distributions are naturally skewed  and the set of observed classes during adaptation may be incomplete.


\item Second, many existing methods depend on sensitive {\em task-specific hyperparameters}, often requiring extensive validation-based grid search to achieve competitive performance ~\cite{ huang2024lp++, clap24}. Such tuning is {\em unrealistic in genuine few-shot settings}, where labeled validation data may be unavailable or costly to obtain. Moreover, relying on large validation ~\cite{lu2022prompt, zhang2022tipadapter, lin2023multimodality} (or test sets) violates the few-shot assumption: In this case, the reported performances depend not only on the adaptation method but also on additional supervision during model selection.
\end{itemize}


Under severe class imbalance, limited minority-class samples could yield poorly discriminative prototypes. As illustrated in Fig.~\ref{fig:geometry_teaser} (left), in CoOp~\cite{zhou2022coop}, the prototypes of ``Industrial'', ``Residential'', and ``Highway'' are concentrated in a narrow region close to the ``Annual Crop'' prototype, despite representing semantically distinct classes. This increases confusion among underrepresented categories. Also, in LP++~\cite{huang2024lp++}, the prototypes are more dispersed, but their locations do not always reflect their corresponding vision-feature distributions. For instance, the ``Permanent Crop'' class prototype lies away from the densest region of its samples. These observations highlight two difficulties under severe imbalance: \emph{insufficient separation} and \emph{unreliable alignment}. \ours,~Fig.~\ref{fig:geometry_teaser} (right), addresses both, better discriminating the crowded minority-class prototypes in CoOp and also aligning each prototypes of LP++ closer to its corresponding visual features. This yields more discriminative and representative class prototypes.

To address these issues, we propose \ours, a validation-free \emph{prototype regularization} framework that explicitly corrects biased and poorly separated class prototypes under realistic few-shot conditions. \ours{} is designed to mitigate the dominance of the majority classes while preserving the discriminative structure of the VLM embedding space.  As illustrated in Fig.~\ref{fig:geometry_teaser} (right), applying \ours{} to both LP++ and CoOp leads to better inter-prototype separation while favoring sample-prototype alignment. In CoOp~\cite{zhou2022coop} (upper row), the minority-class prototypes (`'Industrial'', `'Residential'' and `'Highway'') become more distinctly separated while moving closer to their corresponding visual-feature clusters. Similar effects could be observed for LP++~\cite{huang2024lp++} on a different query set (lower row). These geometric improvements are associated with consistent accuracy gains across diverse adaptation methods, as shown in Fig.~\ref{fig:pannel_teaser}. These findings indicate that prototype bias induced by imbalance acts as a common failure mode across few shot methods, and rectifying it can systematically elevate the performance of diverse approaches.  

Concretely, \ours{} regularizes the prototypes geometry by optimizing a novel loss that combines support-feature alignment, fidelity to the baseline prototypes and pairwise inter-class separation. We optimize our loss with an efficient block majorize–minimize algorithm, alternating between class prototype updates and a lightweight parameter-efficient finetuning of the vision encoder. 

\input{plots/pannel_teaser_fig}

We summarize our contributions as follows:
\begin{itemize}
    \item We expose a critical limitation of standard few-shot protocols for VLMs: class-balanced support and query sets may hide severe prototype bias. We introduce a controlled benchmark that varies both the effective number of classes and the level of class imbalance through Dirichlet sampling.


    \item 
    We introduce \ours{}, a class-prototype regularization that can be deployed as a plug-and-play module on top of any existing baseline method, significantly improving performances (Fig.~\ref{fig:pannel_teaser}). Our method optimizes a novel multi-term loss, which includes a regularizer maximizing inter-class pairwise distances, along with additional terms promoting support-feature alignment and fidelity to the baseline prototypes.

    \item 
    We introduce an effective and computationally efficient block Majorize-Minimize optimizer for our loss. More precisely, we derive a valid blockwise Lipschitz constant  (i.e., a bound on the Hessian’s spectral norm), which can be evaluated effortlessly  via the Gershgorin circle theorem.

    \item We report extensive experiments showing that \ours{} can improve the performance of few-shot baselines, including state-of-the-art methods, with large gains when dealing with severe class imbalance and high numbers of classes.
    \end{itemize}

%% file: plots/geometry_teas_fig.tex
\begin{figure}[t]
    \centering
    \includegraphics[width=\columnwidth]{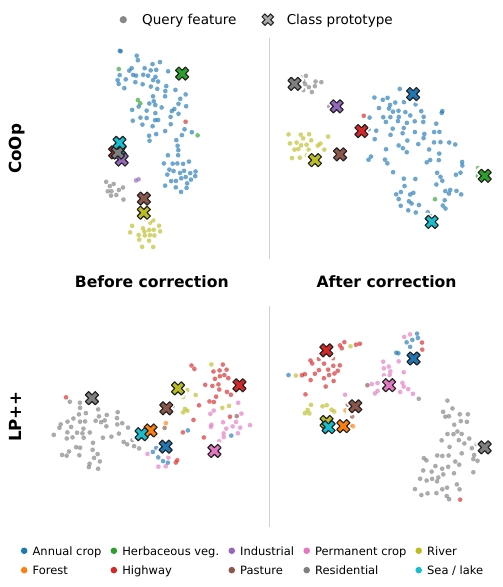}
    \caption{\textbf{Motivation of \ours{}.}
    Visualization on EuroSAT in the 16-shot severe-imbalance setting illustrates how uneven support distributions dominated by majority classes yield biased and poorly separated class prototypes ($\times$) in few-shot adaptation methods, CoOp (top) and LP++ (bottom), increasing inter-class confusion. \ours{} mitigates this bias by promoting prototype orthogonality while encouraging alignment with visual features ($\circ$). Results are shown before (left) and after (right) \ours{} correction.}
    \label{fig:geometry_teaser}
\end{figure}

%% file: plots/pannel_teaser_fig.tex



\begin{figure*}[t]
    \centering
    \includegraphics[width=0.9\textwidth]{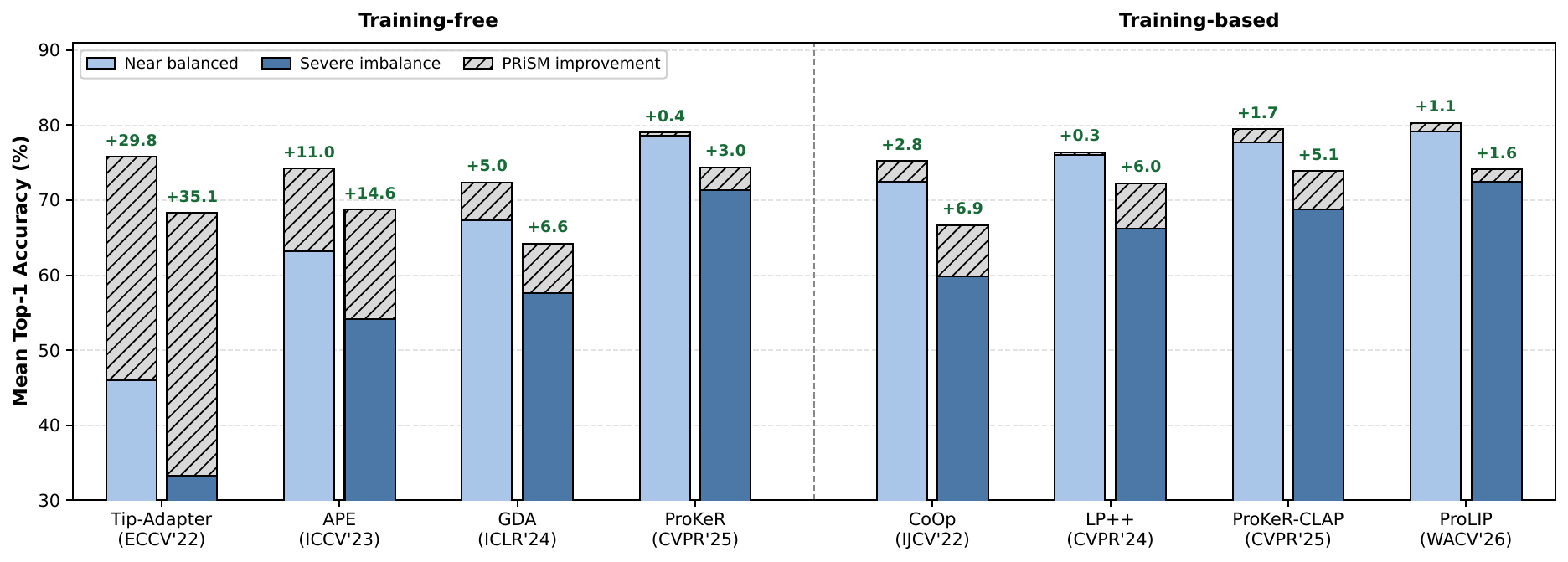}
    \caption{\textbf{Effectiveness of \ours{} across training-free and training-based few-shot adaptation methods.} Mean Top-1 accuracy across 10 datasets for eight baseline methods in the 16-shot setting under near balanced and severely imbalanced class distributions. Hatched regions highlight the performance gains achieved by applying \ours{} to each baseline. \ours{} consistently improves all eight methods, with particularly large gains under severe class imbalance.}
    \label{fig:pannel_teaser}
\end{figure*}

%% file: sec/2_benchmark_description.tex
\section{Toward Realistic Few-Shot Adaptation}
\label{Sec:Benchmark_description}

Let us define a few-shot adaptation task from a downstream dataset with label space
$\mathcal{Y}=\{1,2,\ldots,C\}$. Each task consists of a labeled support set
$\mathcal{D}_{\mathrm{adapt}}=\{(\boldsymbol{x}_n,\boldsymbol{y}_n)\}_{n=1}^{N}$,
where $\boldsymbol{x}_n$ denotes an input image and $\boldsymbol{y}_n$ is its corresponding
one-hot label vector over $\mathcal{Y}$. In the standard few-shot setting,
$\mathcal{D}_{\mathrm{adapt}}$ is constructed as a balanced support set with $K$
labeled examples per class, where $K\in\{1,2,4,8,16\}$. Hence, the total number of
adaptation samples is $N=KC$.

This balanced construction relies on several restrictive assumptions: all classes of
interest are observed during adaptation, each class is represented by the same number
of labeled examples, and the support and query sets follow the same class-frequency
distribution.
These assumptions are often violated in practical few-shot adaptation
scenarios, where the available labeled data may cover only a subset of the target space and may exhibit class proportions that differ from those encountered at evaluation.

To capture these constraints, we consider realistic few-shot adaptation tasks in which
the support set covers only an effective subset of classes
$\mathcal{Y}_{\mathrm{eff}}\subseteq \mathcal{Y}$, with
$C_{\mathrm{eff}}=|\mathcal{Y}_{\mathrm{eff}}|$, and where the support and query
sets may follow different class-frequency distributions. We define two class-coverage
regimes according to the proportion of classes included in
$\mathcal{Y}_{\mathrm{eff}}$: \emph{low} coverage, where $10\%$--$30\%$ of the classes are
available during adaptation, and \emph{high} coverage, where $70\%$--$90\%$ are available.

To model both class imbalance and distributional mismatch between adaptation and
evaluation, we independently sample class proportions for the support and query
sets from Dirichlet distributions over $\mathcal{Y}_{\mathrm{eff}}$. Let
$\boldsymbol{\pi}_{S}\sim\mathrm{Dir}(\delta_{S})$ and
$\boldsymbol{\pi}_{Q}\sim\mathrm{Dir}(\delta_{Q})$ denote the class-proportion vectors
used to construct $\mathcal{D}_{\mathrm{adapt}}$ and
$\mathcal{D}_{\mathrm{query}}$, respectively. Since these proportions are sampled
independently, the class frequencies observed during adaptation may differ from
those encountered at evaluation.

The concentration parameter $\delta\in[0,1]$ controls the dispersion of the sampled
class proportions. Smaller values of $\delta$ produce highly skewed distributions
dominated by a few classes, whereas larger values yield distributions closer to
uniform sampling. We consider two imbalance regimes: severe imbalance, with
$\delta\in[0.1,0.3]$, and near balanced sampling, with $\delta\approx0.9$. For each task,
we fix the support set size to $N=K\times C_{\mathrm{eff}}$ and the query set size
to $N_{\mathrm{query}}=16\times C_{\mathrm{eff}}$ ~\cite{boudiaf2020information}, and then allocate samples across
classes according to $\boldsymbol{\pi}_{S}$ and $\boldsymbol{\pi}_{Q}$. The resulting
class-wise sample counts can vary substantially across both sets and some classes may even be absent in a query set for a given task. However, to ensure that the task remains a valid
few-shot adaptation problem, we enforce that every class in
$\mathcal{Y}_{\mathrm{eff}}$ appears at least once in
$\mathcal{D}_{\mathrm{adapt}}$.

Combining the two class-coverage regimes with the two imbalance regimes yields \textbf{four}
realistic evaluation settings: \emph{low} coverage with \emph{severe} imbalance, \emph{low} coverage with
\emph{near balanced} sampling, \emph{high} coverage with \emph{severe} imbalance, and \emph{high} coverage with
\emph{near balanced} sampling.



%% file: sec/3_methodology.tex
\section{Formulation of \ours{}}
\label{Sec:Methods}
\subsection{Our Method}
Consider a pretrained vision-language model, such as CLIP, composed of an image encoder $f_{\boldsymbol{\phi}}(.)$ and a text encoder $f_{\boldsymbol{\theta}}(.)$, both mapping visual and textual inputs into a shared $d$-dimensional embedding space. Given a downstream task of image classification into $C$ classes, zero-shot prediction amounts to matching image embeddings to text-based class prototypes. For each class $c$, let $\{\boldsymbol{t}_{m,c}\}_{m=1}^M$ denote a set of $M$ textual descriptions (or prompts) of the class, e.g., ``A photo of [class name]''. The corresponding zero-shot prototype is $\boldsymbol{u}_c = \frac{1}{M} \sum_{m} f_{\boldsymbol{\theta}}(\boldsymbol{t}_{m,c}) \in \mathbb{R}^{d}$. Given a query image $\boldsymbol{x}$, with visual embedding $\boldsymbol{v}=f_{\boldsymbol{\phi}}(\boldsymbol{x}) \in \mathbb{R}^{d}$, the zero-shot posterior probability prediction of class $c$ is: 
\begin{equation}
    \mbox{Pr}(c|\boldsymbol{x}) = \frac{\exp(\langle \boldsymbol{v}, \boldsymbol{u}_c \rangle / \tau)}{\sum_{i=1}^{C} \exp(\langle \boldsymbol{v}, \boldsymbol{u}_i \rangle/ \tau)}
    \label{Eq:zero_shot_predictions}
\end{equation}
where pairwise operator $\langle \cdot, \cdot \rangle$ denotes the cosine similarity and $\tau$ is a temperature scaling parameter.

In general, few-shot adaptation methods update the visual features, the class prototypes or both using the support set (i.e., a set of few labeled samples), while discouraging substantial deviations from zero-shot class prototypes 
$\{\boldsymbol{u}_c\}_{c=1}^C$; see, for instance, \cite{clap24,zhang2022tipadapter,bendou2025proker}.  
We write such adaptation mechanisms in a generic form as
$(\mathbf{a}, \boldsymbol{v'}) = f_{\boldsymbol{\psi}}(\mathbf{u}, \boldsymbol{v})$, 
where $\mathbf{u} =[\boldsymbol{u}_1,...,\boldsymbol{u}_{C}] \in \mathbb{R}^{d  C}$ is a vector concatenating all zero-shot prototypes, $\boldsymbol{v'}$ is the adapted visual embedding, and $\boldsymbol{\psi}$ denotes the learnable or method-specific adaptation parameters. We denote by $\mathbf{a} = [\boldsymbol{a}_1, ..., \boldsymbol{a}_{C}] \in \mathbb{R}^{d C}$ a vector concatenating the class prototypes obtained by a given baseline method. Thus, for each class $c$, $\boldsymbol{a}_c$ could be a task-adapted prototype obtained from any given few-shot adaptation method, or just the zero-shot prototype if no adaptation is used. 

As discussed earlier, realistic few-shot settings often exhibit severe class imbalance, causing the support set to provide uneven evidence across the classes. Therefore, some initial prototypes might be less reliable than others: the minority-class prototypes might be weakly aligned with their visual samples, while the majority classes may dominate the classifier decisions, particularly when the support and query (test) distributions differ. To address this issue, our goal is to rectify a classifier $\mathbf{a}$ of a given baseline method by learning regularized prototypes $\mathbf{w} = [\boldsymbol{w}_1,..., \boldsymbol{w}_C] \in \mathbb{R}^{dC}$. For each class $c$, let $\boldsymbol{\mu}_c (\boldsymbol{\phi}) = \frac{1}{n_c} \sum_{n} y_{n,c} f_\mathbf{\phi}(\boldsymbol{x}_n)$ denote the mean of vision features within the class, where $n_c$ is the number of class' samples and $y_{n,c} \in \{0, 1\}$ is a binary-indicator function for the support-set samples: $y_{n,c}=1$ when $c$ is the class of the i$^{\mbox{\small th}}$ support sample and $y_{n,c}=0$ otherwise. We propose to minimize the following objective to learn rectified prototypes:  
\begin{equation}
\resizebox{1\columnwidth}{!}{$\displaystyle
    \mathcal{L}(\mathbf{w}, \boldsymbol{\phi}) = \sum_c \|\boldsymbol{w}_c  - \boldsymbol{a}_c\|_2^2 + \beta \sum_{c} 
   \|\boldsymbol{w}_c  - \boldsymbol{\mu}_c(\boldsymbol{\phi})\|_2^2  \\ - \frac{\lambda}{2(C-1)}\sum_{c\#c'} \|\boldsymbol{w}_c - \boldsymbol{w}_{c'}\|_2^2 
    $}\label{Eq: overall loss}
\end{equation}
where $\|.\|_2$ denotes the standard $l_2$ norm. The first term in 
Eq.~\eqref{Eq: overall loss} anchors our prototype rectification to the baseline-method classifier, discouraging $\boldsymbol{w}_c$ from deviating substantially from $\boldsymbol{a}_c$. The second encourages alignment between each rectified prototype and the corresponding class' visual mean, promoting consistency with the support-set distribution. The last term is the pairwise regularizer driving our prototype rectification: It encourages inter-class separation by pushing away each pair of class prototypes, reducing their interference under our realistic settings. Note that this regularizer is concave with respect to the class prototypes, which makes the overall objective non-convex\footnote{In fact, the convexity of the overall objective in Eq.~\eqref{Eq: overall loss} is controlled by the choice of non-negative weighting factors $\beta$ and $\lambda$. However, choosing these factors such that the objective is strictly convex does not yield the desired prototype rectifications as it makes the $l_2$-regularization components on $\boldsymbol{w}_c$ dominate the overall objective function.} with respect to $\mathbf{w}$. Non-negative scalars $\beta$, and $\lambda$ control the relative strength of each term. 

\subsection{Block Majorize-Minimize Optimization}

Majorize-Minimize (MM) \cite{lange2000optimization} is a general optimization paradigm that includes standard optimizers such as gradient descent, expectation-maximization and concave-convex procedures as special cases. At each iteration, MM minimizes a majorizing function -- an upper bound on the original objective ${\cal L}(\mathbf{w})$ -- which is tight at the current iteration $j$: 
${\cal L}(\mathbf{w}) \leq {\cal M}(\mathbf{w},\mathbf{w}^j)$ and ${\cal L}(\mathbf{w}^j)= {\cal M}(\mathbf{w}^j, \mathbf{w}^j)$. Hence, iterative step $\mathbf{w}^{j+1} = \min_{\mathbf{w}} {\cal M}(\mathbf{w}, \mathbf{w}^j)$ guarantees that the original objective does not increase: ${\cal L}(\mathbf{w}^{j+1}) \leq {\cal M}(\mathbf{w}^{j+1},\mathbf{w}^j) \leq {\cal M}(\mathbf{w}^{j},\mathbf{w}^j) = {\cal L}(\mathbf{w}^{j})$. 
Hence, in the MM framework, step sizes are implicit, unlike standard gradient-descent updates, in which the learning rates are heavily searched using validation data and running the optimizer multiple times. In the following, we take advantage of the blockwise Lipschitz-gradient continuity of our objective in Eq.~\eqref{Eq: overall loss} with respect to the block of prototype variables $\mathbf{w}$. More precisely, we derive a bound on the spectral norm of the blockwise Hessian with respect to $\mathbf{w}$ (Prop.~\ref{Lipschitz-constant}), which yields a block majorizing function with implicit step size. This relaxes the need for computationally intensive validation searches for the step size, while yielding results on par with the best learning rates obtained with validation data; See Fig. \ref{fig:lr} in the experiments. Also, interestingly, our Lipschitz-driven step size is substantially larger than those commonly used in deep learning, yielding steeper descent towards a local minimum. 

Our objective in Eq. \ref{Eq: overall loss} involves two blocks of variables: the corrected prototypes $\mathbf{w}$ and the vision encoder parameters $\boldsymbol{\phi}$. We optimize these variables using a block Majorize-Minimize (MM) procedure, which alternates two MM sub-steps: prototype rectification and feature adaptation. At each iteration, one block is updated while the other is kept fixed. This yields two simpler subproblems: a quadratic problem in $\mathbf{w}$ and a lightweight encoder-adaptation step in $\boldsymbol{\phi}$, with standard gradient steps\footnote{Standard gradient steps could be also viewed as an instance of MM updates.}. The complete procedure is summarized in Algorithm 1 (refer Appendix).

\subsubsection{Prototype Optimization}

When vision-encoder variables $\boldsymbol{\phi}$ are fixed, our objective in Eq.~\eqref{Eq: overall loss} reduces to a quadratic function of $\mathbf{w}$:
\begin{equation}
    \mathcal{L}_w(\mathbf{w}) \stackrel{c}{=} \sum_{c,c'} \mathbf{R}_{cc'} \boldsymbol{w}_{c'}^\top \boldsymbol{w}_c - 2 \sum_c \mathbf{b}_c^\top \boldsymbol{w}_c 
    = \mathbf{w}^\top \mathbf{\Phi}\mathbf{w} -2 \mathbf{B}^\top\mathbf{w}     \label{Eq: quadratic form of our loss}
\end{equation}
where symbol $\stackrel{c}{=}$ means equality up to a constant, 
$\mathbf{b}_c = \beta \boldsymbol{\mu}_c (\boldsymbol{\phi}) + \boldsymbol{a}_c \in \mathbb{R}^{d}$, and $\mathbf{R} \in \mathbb{R}^{C \times C}$ is a prototype-interaction matrix given by $\mathbf{R}_{cc'} = \frac{\lambda}{C-1}$ for $c \neq c'$ and $\mathbf{R}_{cc} = \beta + 1 - \lambda$. 
Also, matrix $\mathbf{B} = [\mathbf{b}_1,...,\mathbf{b}_C] \in \mathbb{R}^{dC \times dC}$, and matrix $\mathbf{\Phi} = \mathbf{R} \otimes \mathbf{I}_d$, where $\otimes$ denotes the Kronecker product and $\mathbf{I}_d$ is the $d \times d$ identity matrix. 

As $\mathcal{L}_w(\mathbf{w})$ in Eq.~\eqref{Eq: quadratic form of our loss} is quadratic, twice-differentiable function, it has a Lipschitz continuous gradient, i.e., there exists a 
strictly positive Lipschitz constant $\rho$ such that the spectral norm (the largest singular value) of the Hessian is bounded by $\rho$: $\|\nabla^2 \mathcal{L}_w(\mathbf{w})\|_2 \leq \rho ~ \forall \mathbf{w}$, with $\|.\|_2$ denoting the spectral norm for a matrix.   
From this inequality, we obtain a majorizing function for $\mathcal{L}_w$ at iteration $j$:
\begin{equation}
\resizebox{1\columnwidth}{!}{$\displaystyle
\mathcal{L}_w (\mathbf{w}) \leq \mathcal{M}_{w} (\mathbf{w},\mathbf{w}^j) = \mathcal{L}_w(\mathbf{w}^j)+
    \nabla\mathcal{L}_{w}(\mathbf{w}^j)^\top (\mathbf{w}-\mathbf{w}^j) + \frac{\rho}{2}
    \|\mathbf{w}-\mathbf{w}^j\|_2^2 $}
    \label{Eq: prototypes_majorizer}
\end{equation}
Minimizing the majorizing function $\mathcal{M}_{w}$ in Eq.~\eqref{Eq: prototypes_majorizer} yields gradient-type updates with implicit step sizes $1/\rho$, which guarantee that the objective decreases by at least $\frac{1}{2\rho} \|\nabla \mathcal{L}(\mathbf{w}^j)\|^2$ at each iteration, a result known as 
the Descent Lemma. One valid Lipschitz constant would be the maximum singular value of the Hessian of Eq.~\eqref{Eq: quadratic form of our loss}, i.e. twice the spectral norm of $\mathbf{\Phi}$: $2\|\mathbf{\Phi}\|_2$. However, a spectral decomposition of $\mathbf{\Phi}$, a $dC \times dC$ matrix, could be computationally intensive for large numbers of classes. In the following proposition, we derive an explicit upper bound on $2\|\mathbf{\Phi}\|_2$, which provides a valid Lipschitz constant. 

\begin{proposition}
\label{Lipschitz-constant}
Objective $\mathcal{L}_w$ in Eq.~\eqref{Eq: quadratic form of our loss} 
has $\rho$-Lipschitz gradient, with $\rho$ given by:
\begin{equation}
    \rho = 2 (|\beta + 1 - \lambda| + \lambda) \geq  2 \|\mathbf{\Phi}\|_2    \label{Eq: Lp constant prototypes}
\end{equation}
\label{p1}
\end{proposition}
\begin{proof}
The proof follows from the fact that $\|\mathbf{\Phi}\|_2 = \| \mathbf{R}\otimes \mathbf{I}_d \|_2 = \| \mathbf{R} \|_2  \| \mathbf{I}_d \|_2 = \| \mathbf{R} \|_2$, and from bounding $\| \mathbf{R} \|_2$ using the Gershgorin’s circle theorem. The details are deferred to the appendix.
\end{proof}

\subsection{Vision Encoder Optimization}

For the vision encoder block, with $\mathbf{w}$ fixed, the objective depends only on $\boldsymbol{\phi}$  through the visual mean $\boldsymbol{\mu}(\boldsymbol{\phi})$. We optimize this block by standard back-propagation using a conventional learning rate $\boldsymbol{\eta}_{\boldsymbol{\phi}}$, while restricting the trainable parameters to a parameter-efficient subset of the image encoder. Specifically, we use LoRA \cite{hu2022lora} to introduce a small set of trainable low-rank updates in selected layers of $f_{\boldsymbol{\phi}}(.)$. For a frozen pretrained weight matrix $\mathbf{T}_0$, LoRA parameterizes the adapted weight as $\mathbf{T}=\mathbf{T}_0+\mathbf{PL}$ where $\mathbf{P}\in\mathbb{R}^{l1_\times r}$ and $\mathbf{L}\in\mathbb{R}^{r\times l2}$ are trainable low-rank factors with $r\ll min(l1, l2)$. Only $\mathbf{P}$ and $\mathbf{L}$ are optimized, while $\mathbf{T}_0$ remains fixed. This parameter-efficient design reduces the number of trainable parameters while maintaining the expressive power of the pretrained vision encoder.

%% file: sec/4_experiments.tex
\section{Experiments}
\label{experiments}
\noindent\textbf{Datasets.}
Following standard practices in vision-language models~\cite{radford2021clip,zhou2022coop,zhou2022cocoop,zanella2024boosting}, we consider $10$ widely used image classification benchmarks covering diverse domains and levels of granularity. These include fine-grained recognition datasets such as Oxford-Pets~\cite{parkhi2012pets}, Stanford Cars~\cite{krause2013stanfordcars}, and FGVC Aircraft~\cite{maji2013fgvcaircraft}; texture and material datasets including Describable Textures Dataset (DTD)~\cite{cimpoi2014dtd} and Food-101~\cite{bossard2014food101}; scene and remote sensing benchmarks such as SUN397~\cite{xiao2010sun397} and EuroSAT~\cite{helber2019eurosat}; as well as generic object recognition datasets including Oxford Flowers~\cite{nilsback2008flowers102}, Caltech-101~\cite{fei2004caltech101},and UCF101~\cite{soomro2012ucf101}. To avoid task-specific validation on the downstream benchmarks, we reserve ImageNet ~\cite{deng2009imagenet} exclusively as a held-out validation dataset for selecting a single hyperparameter configuration applied to all evaluation datasets. 

\noindent\textbf{Baselines.}
We evaluate \ours{} on top of eight representative few-shot adaptation methods for VLMs, covering diverse adaptation mechanisms. The \textit{training-free} baselines include the cache-based retrieval method Tip Adapter ~\cite{zhang2022tipadapter}; Gaussian Discriminant Analysis (GDA) ~\cite{wang2024baseline}; Adaptive Prior rEfinement (APE) ~\cite{zhu2023not}; and the kernel-based prototype refinement method ProKeR ~\cite{bendou2025proker}. The \textit{training-based} baselines include CoOp ~\cite{zhou2022coop}, which learns task-specific context tokens to adapt text prompts; ProKeR-CLAP ~\cite{bendou2025proker} , which jointly optimizes a linear classifier under the CLAP \cite{clap24} objective and the kernel-based solution of ProKeR; LP++ ~\cite{huang2024lp++}, which learns a linear classifier combining visual features and textual priors; and ProLIP ~\cite{fahes2026clip}, which fine-tunes the last vision encoder's projection layer with regularization toward its pretrained weights.  For LP++, we consider two loss variants: LP++(CE), using cross-entropy, and LP++(FL), using focal loss to mitigate class imbalance. This diversity allows us to assess the generality of \ours{} as a plug-and-play prototype rectification module that can be applied across different few shot adaptation paradigms.
\input{tables/severe_high_vit16}
\noindent\textbf{Implementation Details.} All experiments use CLIP ~\cite{radford2021clip} with a ViT-B/16 backbone ~\cite{dosovitskiy2020image}. For each downstream task, we first adapt the corresponding baseline method to the task-specific support set, followed by our prototype rectification framework  applied to the resulting classifier. We adopt the dataset-specific text prompts from~\cite{bendou2025proker} and report results averaged over $400$ randomly sampled few-shot tasks. To maintain a validation-free setting, we use the same hyperparameter configuration across all datasets and shot settings, with $\beta=0.01$, $\lambda=0.05$ and a fixed optimization budget of $33$ update steps including the prototype and vision encoder blocks. For the latter, we employ rank-$8$ LoRA adapters with a learning rate of $5 \times 10^{-4}$.All experiments are conducted on a single NVIDIA H100 SXM5 GPU. Additional implementation details are provided in the Appendix.

\subsection{Severe Imbalance with High Effective Classes}

\noindent\textbf{Shot Budget under Class Imbalance.} Table~\ref{tab:severe_high_eff} reveals the vulnerability of few-shot adaptation methods to severe class imbalance. With $C_{\mathrm{eff}}$ fixed at $80\%$ of the total classes, increasing the shot budget $K$ introduces more support samples but does not alleviate the skewed Dirichlet distribution. As a result, majority classes become increasingly dominant, while minority classes remain under-represented, leading to progressively biased prototype estimates (see Fig. \ref{fig:geometry_teaser} (left)). This effect is clearly reflected in the performance degradation of several baselines when moving from 4 to 16 shots. Notably, Tip-Adapter and APE suffer substantial performance drops from $49.64$ to $33.25$ and from $63.58$ to $54.15$, respectively. Similar trends are observed in training-based methods, with LP++(CE) declining from $68.26$ to $66.26$ and ProKeR+CLAP from $70.01$ to $68.81$. Even competitive methods such as ProKeR ($71.56 \rightarrow 71.40$) and ProLIP ($72.62 \rightarrow 72.48$) fail to benefit from additional shots. 
These findings confirm that increasing the support budget alone does not guarantee improved adaptation under class imbalance.
This is, in fact, consistent with the following technical result:
{\em For a fixed nonuniform Dirichlet draw $\boldsymbol{\pi}_S$,
the class proportions $\widetilde{\boldsymbol{\pi}}_S^{(K)}$ induced by a valid few-shot support allocation deviate
strictly further from uniformity as the shot budget $K$ increases:
\[\left\|\widetilde{\boldsymbol{\pi}}_S^{(K)}
-\boldsymbol{\pi}_U\right\|_2^2
=
\left(1-\frac{1}{K}\right)^2
\left\|\boldsymbol{\pi}_S-\boldsymbol{\pi}_U\right\|_2^2 \]
where $\boldsymbol{\pi}_U=\mathbf{1}/C_{\mathrm{eff}}$
denotes the uniform distribution.}

The proof follows from
$\widetilde{\boldsymbol{\pi}}_S^{(K)}
=\boldsymbol{\pi}_U/K+(1-1/K)\boldsymbol{\pi}_S$.
Further details are deferred to the appendix.

\noindent\textbf{Effectiveness of \ours{} across Adaptation Strategies.} In contrast, \ours{} addresses this limitation through prototype rectification, yielding consistent accuracy improvements that become more pronounced as the shot budget increases. Among training-free methods, APE gains $\boldsymbol{+7.65}$ and $\boldsymbol{+14.59}$ at 4 and 16 shots, respectively, while GDA improves by $\boldsymbol{+4.05}$ and $\boldsymbol{+6.56}$. These benefits extend to the prompt-learning method CoOp, with gains increasing from $\boldsymbol{+5.63}$ to $\boldsymbol{+6.86}$, and to learning-based methods such as LP++(CE) ($\boldsymbol{+3.41}$ to $\boldsymbol{+6.03}$) and ProKeR-CLAP ($\boldsymbol{+2.47}$ to $\boldsymbol{+5.08}$). As illustrated in Fig.~\ref{fig:geometry_teaser}, \ours{} promotes greater inter-class prototype separation and improved alignment with corresponding visual features in both CoOp and LP++, providing a geometric interpretation of these gains. Strong baselines such as ProLIP and ProKeR also benefit, achieving improvements of $\boldsymbol{+1.62}$ and $\boldsymbol{+3.03}$ at 16 shots, respectively. These results highlight \ours{} as an effective plug-and-play correction that regularizes prototype geometry while preserving semantic alignment under severe class imbalance.

\noindent\textbf{\ours{} on Top of Imbalance-Aware Training.} We further evaluate \ours{} on top of focal loss, a widely used objective for mitigating class imbalance. While LP++(FL) outperforms LP++(CE) at both 4 shots ($69.28$ vs. $68.26$) and 16 shots ($68.16$ vs. $66.26$), its performance still declines as the shot budget increases. Incorporating \ours{} yields further improvements of $\boldsymbol{+1.85}$ and $\boldsymbol{+3.80}$, reaching $\boldsymbol{71.13}$ and $\boldsymbol{71.96}$ at 4 and 16 shots, respectively. 

\subsection{Severe Imbalance with Low Effective Classes}
\input{tables/severe_low_vit16_means}
In table \ref{tab:severe_low_eff_means} we report results for low effective classes in severe imbalance setting. In this regime, the adaptation problem is intrinsically easier, as the support set covers a reduced label space, leading to more reliable class-wise statistics even under imbalance. Consequently, baseline methods exhibit significantly less degradation compared to the high effective class setting, as the shot budget increases. For instance, when moving from $4$ to $16$ shots, ProKeR declines slightly from $84.64$ to $83.74$, while LP++(CE) decreases from $82.21$ to $81.28$ and  ProLIP from $86.05$ to $85.02$. Although this regime is easier, \ours{} yields consistent improvements across diverse adaptation strategies. Among training-free methods, APE gains $\boldsymbol{+3.36}$ and $\boldsymbol{+7.22}$ at 4 and 16 shots, respectively, while Proker improves by $\boldsymbol{+1.66}$ and $\boldsymbol{+3.59}$. Similar gains are observed for the prompt-learning method CoOp ($\boldsymbol{+5.93}$ and $\boldsymbol{+7.06}$) and learning-based methods such as LP++(CE) ($\boldsymbol{+3.24}$ and $\boldsymbol{+5.14}$) and ProKeR+CLAP ($\boldsymbol{+1.96}$ and $\boldsymbol{+4.38}$). Even for competitive baselines such as ProLIP, \ours{} still yields additional improvements, particularly in the higher-shot regime (e.g., $\boldsymbol{+1.67}$ at $16$ shots). This indicates that despite the limited number of effective classes, imbalance still introduces biases in prototype estimation that accumulate with increasing shots, highlighting the effectiveness of \ours{} in improving few-shot adaptation baselines across different numbers of effective classes.

\subsection{Evaluation under the Near Balanced Setting}
\input{tables/bal_low_high_vit16_means}
In table \ref{tab:balanced_low_high_eff_mean} we present results for high and low effective number of classes under near balanced class distributions, where the Dirichlet concentration is high ($0.9$) and class proportions are relatively uniform. Unlike the severe imbalance regime, most baselines, particularly training-based methods, benefit from increasing the shot budget, as more uniform support distributions enable more reliable prototype estimation. Under high effective classes, CoOp improves from $65.28$ to $72.49$ and LP++(CE) from $73.32$ to $76.06$ when moving from 4 to 16 shots. Similar trends are observed for GDA ($61.40 \rightarrow 67.38$) and ProKeR ($75.36 \rightarrow 78.65$). However, other training-free methods such as Tip-Adapter and APE still exhibit performance degradation as the shot budget increases, suggesting that challenges associated with class imbalance may persist even in near balanced settings. Despite this more favorable regime, \ours{} further improves most baselines, highlighting the effectiveness of prototype rectification even when class imbalance is less pronounced. 
At 16 shots under high effective classes, APE and GDA achieve gains of $\boldsymbol{+10.99}$ and $\boldsymbol{+4.96}$, respectively. Similar benifits are observed for the prompt learning method CoOp ($\boldsymbol{+2.78}$) and learning-based approaches such as ProKeR+CLAP ($\boldsymbol{+1.70}$) and ProLIP ($\boldsymbol{+1.12}$).

\subsection{Ablation Studies}
\input{plots/lr_ablation}
\noindent\textbf{Sensitivity to Lipschitz-Based Learning Rate.}
In this section we analyze the impact of Lipschitz-based learning rate on performance. As discussed earlier, the proposed prototype rectification objective admits a quadratic form with a closed-form upper bound on its curvature. This allows us to derive a valid block Lipschitz constant $\rho$, which naturally induces a principled learning rate of $1 / \rho $ for the prototype update. Unlike conventional hyperparameter tuning, this choice is directly tied to the geometry of the optimization landscape and guarantees stable MM updates.
We conduct this analysis on top of ProKeR~\cite{bendou2025proker} under the 16-shot severe imbalance setting with a high number of effective classes, using EuroSat, DTD and Flowers as representative datasets. To isolate the effect of the prototype-block step size, we vary only the prototype learning rate while keeping all other hyperparameters fixed. As observed from Fig.~\ref{fig:lr}, performance exhibits a clear sensitivity to the learning rate. Extremely small learning rates under-optimize the prototype correction, leading to weaker alignment between the rectified prototypes and the support-set feature distribution. In contrast, the Lipschitz-based learning rate, indicated by the vertical dashed line, consistently matches the best-performing region across datasets. This indicates that the proposed step-size selection captures the effective curvature of the prototype rectification objective and provides a reliable operating point without requiring an exhaustive learning-rate sweep. Notably, the resulting step size is orders of magnitude larger than those typically used in deep learning, usually lying in the range $[10^{-4}, 10^{-2}]$; see ~\cite{zhang2022tipadapter}.

\noindent\textbf{Sensitivity to Orthogonality Weight $\lambda$.} Fig. \ref{fig:combined_ablation} (a, right) examines the effect of the orthogonality weight $\lambda$ from Eq.~\ref{Eq: overall loss} on \ours{} applied to LP++(CE) in the 16-shot severe imbalance setting with high effective number of classes across 10 datasets. Removing the orthogonality term ($\lambda=0$) reduces mean accuracy to $69.1$, compared to approximately $72.3$ with the default $\lambda=0.05$. Further increasing $\lambda$ to $0.5$ and $5$ gradually degrades performance, suggesting that excessive inter-class separation may inhibit prototype alignment with their corresponding visual-feature distributions.

\input{plots/prism_ablations}

\noindent\paragraph{Sensitivity to Alignment Weight $\beta$.} Fig. \ref{fig:combined_ablation} (a, left) investigates the influence of the alignment weight $\beta$ under the same setting. Small values ($\beta=0.01$ and $0.1$) maintain mean accuracies above $72$, compared to $70.0$ without alignment regularization ($\beta=0$). However, increasing $\beta$ to $1$ and $10$ substantially reduces accuracy to approximately $68.3$ and $67.3$, respectively, indicating that excessive alignment restricts prototype rectification. Together, these ablations demonstrate the robustness of \ours{} and highlight the need to balance visual-feature alignment with inter-class separation. 

\noindent\textbf{\ours{} as Standalone Few Shot Adaptation Method.} We further evaluate \ours{} as a standalone few-shot adaptation method applied directly to pretrained CLIP under severe class imbalance with a high number of effective classes. As shown in Fig. \ref{fig:combined_ablation}, \ours{} improves upon the zero-shot CLIP baseline ($67.5$) by $\boldsymbol{+7.3}$ and $\boldsymbol{+9.1}$ at 4 and 16 shots, respectively, under severe class imbalance with high number of effective classes. These gains demonstrate that \ours{} can yield competitive performance without requiring auxiliary adaptation mechanisms.

%% file: tables/severe_high_vit16.tex
\begin{table*}[!htbp]
\centering
\small
\setlength{\tabcolsep}{1pt}
\renewcommand{\arraystretch}{1.08}
\resizebox{0.85\textwidth}{!}{%
\scalebox{1}[0.8]{%
\begin{tabular}{ll|cccccccccc|c}
\toprule
 && \textbf{Caltech} & \textbf{DTD} & \textbf{EuroSAT} & \textbf{FGVC} & \textbf{Food101} & \textbf{Flowers} & \textbf{Pets} & \textbf{Cars} & \textbf{Sun397} & \textbf{UCF101} & \textbf{Mean} \\
\midrule

\multicolumn{13}{c}{\textbf{16-shot}} \\
\midrule


\multirow{10}{*}{\rotatebox{90}{\shortstack{\textbf{Training Based}}}}

&CoOp
& 83.16 & 46.00 & 60.53 & 29.07 & 61.24 & 78.18 & 75.02 & 54.19 & 49.90 & 61.20 & 59.85 \\
& \gc w/ \ours
& \gc \textbf{88.03\imp{+4.87}}
& \gc \textbf{55.74\imp{+9.74}}
& \gc \textbf{69.30\imp{+8.77}}
& \gc \textbf{30.92\imp{+1.85}}
& \gc \textbf{71.16\imp{+9.92}}
& \gc \textbf{84.56\imp{+6.37}}
& \gc \textbf{81.23\imp{+6.21}}
& \gc \textbf{59.40\imp{+5.22}}
& \gc \textbf{58.02\imp{+8.12}}
& \gc \textbf{68.71\imp{+7.51}}
& \gc \textbf{66.71\imp{+6.86}} \\

&LP++$_\textbf{{(CE)}}$
& 90.98 & 47.88 & 59.13 & 30.73 & 77.41 & 80.81 & 82.96 & 63.06 & 60.81 & 68.85 & 66.26 \\
& \gc w/ \ours
& \gc \textbf{93.19\imp{+2.21}}
& \gc \textbf{60.76\imp{+12.88}}
& \gc \textbf{75.31\imp{+16.18}}
& \gc \textbf{32.29\imp{+1.57}}
& \gc \textbf{83.07\imp{+5.66}}
& \gc \textbf{80.49\wor{-0.32}}
& \gc \textbf{89.58\imp{+6.63}}
& \gc \textbf{64.82\imp{+1.76}}
& \gc \textbf{67.48\imp{+6.67}}
& \gc \textbf{75.93\imp{+7.07}}
& \gc \textbf{72.29\imp{+6.03}} \\

&LP++$_\textbf{{(FL)}}$
& 91.63 & 48.67 & 63.51 & 31.44 & 78.86 & 82.04 & 85.80 & 65.96 & 62.86 & 70.82 & 68.16 \\
& \gc w/ \ours
& \gc \textbf{93.18\imp{+1.55}}
& \gc \textbf{59.81\imp{+11.14}}
& \gc \textbf{74.67\imp{+11.16}}
& \gc \textbf{31.51\imp{+0.06}}
& \gc \textbf{83.31\imp{+4.45}}
& \gc \textbf{78.82\wor{-3.22}}
& \gc \textbf{88.73\imp{+2.93}}
& \gc \textbf{65.51\wor{-0.46}}
& \gc \textbf{67.95\imp{+5.10}}
& \gc \textbf{76.12\imp{+5.31}}
& \gc \textbf{71.96\imp{+3.80}} \\

&ProKeR + CLAP
& 91.66 & 54.47 & 65.25 & 32.54 & 79.20
& 85.43 & 79.44 & 63.96 & 64.25 & 71.89 & 68.81 \\

& \gc w/ \ours
& \gc \textbf{93.09\imp{+1.43}}
& \gc \textbf{64.51\imp{+10.04}}
& \gc \textbf{78.63\imp{+13.38}}
& \gc \textbf{34.41\imp{+1.87}}
& \gc \textbf{81.58\imp{+2.38}}
& \gc \textbf{89.09\imp{+3.66}}
& \gc \textbf{88.88\imp{+9.44}}
& \gc \textbf{66.99\imp{+3.03}}
& \gc \textbf{66.05\imp{+1.80}}
& \gc \textbf{75.69\imp{+3.80}}
& \gc \textbf{73.89\imp{+5.08}} \\

&ProLIP
& 92.94 & 55.71 & 68.24 & 34.08 & 83.31
& 87.41 & 87.10 & 70.38 & 69.55 & 76.10 & 72.48 \\

& \gc w/ \ours
& \gc \textbf{92.95\imp{+0.01}}
& \gc \textbf{61.75\imp{+6.05}}
& \gc \textbf{77.77\imp{+9.52}}
& \gc \textbf{35.85\imp{+1.76}}
& \gc \textbf{80.82\wor{-2.49}}
& \gc \textbf{91.45\imp{+4.05}}
& \gc \textbf{88.53\imp{+1.44}}
& \gc \textbf{68.09\wor{-2.29}}
& \gc \textbf{66.61\wor{-2.94}}
& \gc \textbf{77.23\imp{+1.13}}
& \gc \textbf{74.10\imp{+1.62}} \\

\multicolumn{13}{@{}l@{}}{\dashfill} \\


\multirow{8}{*}{\rotatebox{90}{\textbf{Training Free}}}

&Tip-Adapter
& 82.51 & 21.56 & 21.05 & 17.71 & 19.64 & 20.06 & 21.51 & 35.24 & 48.73 & 44.53 & 33.25 \\
& \gc w/ \ours
& \gc \textbf{93.24\imp{+10.73}}
& \gc \textbf{61.76\imp{+40.21}}
& \gc \textbf{74.36\imp{+53.32}}
& \gc \textbf{31.40\imp{+13.69}}
& \gc \textbf{33.92\imp{+14.29}}
& \gc \textbf{84.98\imp{+64.92}}
& \gc \textbf{91.04\imp{+69.53}}
& \gc \textbf{68.10\imp{+32.86}}
& \gc \textbf{68.20\imp{+19.47}}
& \gc \textbf{76.03\imp{+31.50}}
& \gc \textbf{68.30\imp{+35.05}} \\

&APE
& 91.67 & 50.77 & 34.10 & 19.26 & 60.34 & 41.50 & 69.25 & 44.17 & 62.08 & 68.38 & 54.15 \\
& \gc w/ \ours
& \gc \textbf{92.80\imp{+1.13}}
& \gc \textbf{61.20\imp{+10.44}}
& \gc \textbf{76.12\imp{+42.02}}
& \gc \textbf{26.83\imp{+7.57}}
& \gc \textbf{76.35\imp{+16.01}}
& \gc \textbf{69.73\imp{+28.23}}
& \gc \textbf{86.82\imp{+17.57}}
& \gc \textbf{56.75\imp{+12.58}}
& \gc \textbf{65.74\imp{+3.66}}
& \gc \textbf{75.08\imp{+6.71}}
& \gc \textbf{68.74\imp{+14.59}} \\

&GDA
& 87.68 & 44.33 & 58.34 & 23.08 & 55.80 & 63.59 & 66.29 & 50.72 & 60.31 & 66.26 & 57.64 \\
& \gc w/ \ours
& \gc \textbf{88.11\imp{+0.43}}
& \gc \textbf{53.50\imp{+9.17}}
& \gc \textbf{58.47\imp{+0.12}}
& \gc \textbf{28.47\imp{+5.39}}
& \gc \textbf{68.79\imp{+12.99}}
& \gc \textbf{79.77\imp{+16.18}}
& \gc \textbf{74.55\imp{+8.26}}
& \gc \textbf{57.41\imp{+6.69}}
& \gc \textbf{64.57\imp{+4.25}}
& \gc \textbf{68.39\imp{+2.13}}
& \gc \textbf{64.20\imp{+6.56}} \\

&ProKeR
& 91.86 & 57.39 & 65.61 & 34.71 & 82.47 & 86.91 & 84.18 & 68.39 & 67.91 & 74.54 & 71.40 \\
& \gc w/ \ours
& \gc \textbf{92.80\imp{+0.94}}
& \gc \textbf{62.75\imp{+5.35}}
& \gc \textbf{80.65\imp{+15.04}}
& \gc \textbf{35.13\imp{+0.42}}
& \gc \textbf{81.51\wor{-0.97}}
& \gc \textbf{88.94\imp{+2.04}}
& \gc \textbf{90.10\imp{+5.92}}
& \gc \textbf{68.38\wor{-0.01}}
& \gc \textbf{67.50\wor{-0.41}}
& \gc \textbf{76.51\imp{+1.97}}
& \gc \textbf{74.43\imp{+3.03}} \\

\midrule
\multicolumn{13}{c}{\textbf{4-shot}} \\
\midrule


\multirow{10}{*}{\rotatebox{90}{\shortstack{\textbf{Training Based}}}}

&CoOp
& 82.32 & 42.29 & 56.05 & 23.91 & 64.51 & 78.64 & 76.73 & 51.30 & 47.18 & 58.56 & 58.15 \\
& \gc w/ \ours
& \gc \textbf{86.84\imp{+4.52}}
& \gc \textbf{52.21\imp{+9.92}}
& \gc \textbf{66.14\imp{+10.09}}
& \gc \textbf{26.08\imp{+2.18}}
& \gc \textbf{69.56\imp{+5.05}}
& \gc \textbf{83.18\imp{+4.54}}
& \gc \textbf{79.75\imp{+3.02}}
& \gc \textbf{54.90\imp{+3.60}}
& \gc \textbf{53.52\imp{+6.34}}
& \gc \textbf{65.62\imp{+7.07}}
& \gc \textbf{63.78\imp{+5.63}} \\

&LP++$_\textbf{{(CE)}}$
& 91.82 & 48.12 & 62.65 & 28.86 & 81.24 & 84.38 & 85.54 & 64.88 & 64.02 & 71.04 & 68.26 \\
& \gc w/ \ours
& \gc \textbf{93.02\imp{+1.20}}
& \gc \textbf{58.28\imp{+10.15}}
& \gc \textbf{71.98\imp{+9.33}}
& \gc \textbf{31.21\imp{+2.34}}
& \gc \textbf{83.30\imp{+2.06}}
& \gc \textbf{81.29\wor{-3.09}}
& \gc \textbf{88.92\imp{+3.38}}
& \gc \textbf{65.32\imp{+0.44}}
& \gc \textbf{67.66\imp{+3.65}}
& \gc \textbf{75.66\imp{+4.62}}
& \gc \textbf{71.66\imp{+3.41}} \\

&LP++$_\textbf{{(FL)}}$
& 91.92 & 49.45 & 65.09 & 28.94 & 81.61 & 83.42 & 88.07 & 66.90 & 65.32 & 72.12 & 69.28 \\
& \gc w/ \ours
& \gc \textbf{92.99\imp{+1.07}}
& \gc \textbf{58.74\imp{+9.29}}
& \gc \textbf{73.41\imp{+8.32}}
& \gc \textbf{29.76\imp{+0.82}}
& \gc \textbf{82.86\imp{+1.25}}
& \gc \textbf{78.35\wor{-5.06}}
& \gc \textbf{87.31\wor{-0.76}}
& \gc \textbf{64.60\wor{-2.30}}
& \gc \textbf{67.87\imp{+2.55}}
& \gc \textbf{75.41\imp{+3.29}}
& \gc \textbf{71.13\imp{+1.85}} \\

&ProKeR + CLAP

& 92.32 & 53.65 & 67.01 & 30.14 & 82.27
& 86.20 & 84.96 & 64.36 & 66.17 & 73.02 & 70.01 \\

& \gc w/ \ours
& \gc \textbf{92.75\imp{+0.43}}
& \gc \textbf{60.42\imp{+6.77}}
& \gc \textbf{77.37\imp{+10.36}}
& \gc \textbf{32.25\imp{+2.11}}
& \gc \textbf{80.65\wor{-1.62}}
& \gc \textbf{88.47\imp{+2.27}}
& \gc \textbf{88.50\imp{+3.54}}
& \gc \textbf{65.11\imp{+0.75}}
& \gc \textbf{64.50\wor{-1.67}}
& \gc \textbf{74.83\imp{+1.81}}
& \gc \textbf{72.48\imp{+2.47}} \\

&ProLIP
& 93.37 & 55.33 & 66.63 & 33.07 & 85.43
& 85.69 & 89.24 & 71.07 & 70.79 & 75.62 & 72.62 \\

& \gc w/ \ours
& \gc \textbf{92.77\wor{-0.60}}
& \gc \textbf{59.70\imp{+4.37}}
& \gc \textbf{72.74\imp{+6.11}}
& \gc \textbf{34.05\imp{+0.98}}
& \gc \textbf{80.64\wor{-4.79}}
& \gc \textbf{90.43\imp{+4.74}}
& \gc \textbf{89.20\wor{-0.04}}
& \gc \textbf{65.97\wor{-5.10}}
& \gc \textbf{65.97\wor{-4.82}}
& \gc \textbf{75.37\wor{-0.25}}
& \gc \textbf{72.68\imp{+0.06}} \\

\multicolumn{13}{@{}l@{}}{\dashfill} \\


\multirow{8}{*}{\rotatebox{90}{\textbf{Training Free}}}

&Tip-Adapter
& 89.58 & 34.70 & 27.91 & 16.32 & 72.07 & 26.43 & 57.81 & 47.83 & 59.91 & 63.88 & 49.64 \\
& \gc w/ \ours
& \gc \textbf{92.97\imp{+3.39}}
& \gc \textbf{59.89\imp{+25.19}}
& \gc \textbf{78.12\imp{+50.21}}
& \gc \textbf{31.00\imp{+14.69}}
& \gc \textbf{82.33\imp{+10.27}}
& \gc \textbf{84.24\imp{+57.81}}
& \gc \textbf{90.41\imp{+32.61}}
& \gc \textbf{67.12\imp{+19.30}}
& \gc \textbf{67.26\imp{+7.35}}
& \gc \textbf{75.27\imp{+11.39}}
& \gc \textbf{72.86\imp{+23.22}} \\

&APE
& 92.11 & 49.91 & 52.32 & 19.19 & 84.13 & 58.52 & 86.95 & 55.07 & 65.92 & 71.68 & 63.58 \\
& \gc w/ \ours
& \gc \textbf{92.56\imp{+0.45}}
& \gc \textbf{59.75\imp{+9.84}}
& \gc \textbf{77.22\imp{+24.90}}
& \gc \textbf{28.14\imp{+8.94}}
& \gc \textbf{81.12\wor{-3.01}}
& \gc \textbf{80.75\imp{+22.23}}
& \gc \textbf{89.73\imp{+2.78}}
& \gc \textbf{61.62\imp{+6.55}}
& \gc \textbf{66.18\imp{+0.26}}
& \gc \textbf{75.27\imp{+3.59}}
& \gc \textbf{71.23\imp{+7.65}} \\

&GDA
& 86.12 & 45.04 & 58.32 & 23.23 & 54.52 & 68.29 & 62.83 & 45.71 & 54.45 & 64.43 & 56.30 \\
& \gc w/ \ours
& \gc \textbf{86.72\imp{+0.60}}
& \gc \textbf{49.96\imp{+4.92}}
& \gc \textbf{57.22\wor{-1.11}}
& \gc \textbf{26.51\imp{+3.27}}
& \gc \textbf{61.54\imp{+7.03}}
& \gc \textbf{77.02\imp{+8.73}}
& \gc \textbf{71.08\imp{+8.25}}
& \gc \textbf{50.68\imp{+4.97}}
& \gc \textbf{57.10\imp{+2.65}}
& \gc \textbf{65.67\imp{+1.23}}
& \gc \textbf{60.35\imp{+4.05}} \\

&ProKeR
& 92.21 & 56.42 & 66.09 & 32.38 & 84.84 & 86.20 & 86.21 & 68.44 & 68.17 & 74.63 & 71.56 \\
& \gc w/ \ours
& \gc \textbf{92.62\imp{+0.41}}
& \gc \textbf{60.53\imp{+4.12}}
& \gc \textbf{79.01\imp{+12.92}}
& \gc \textbf{33.42\imp{+1.04}}
& \gc \textbf{81.16\wor{-3.68}}
& \gc \textbf{87.64\imp{+1.44}}
& \gc \textbf{88.93\imp{+2.72}}
& \gc \textbf{66.49\wor{-1.95}}
& \gc \textbf{66.15\wor{-2.01}}
& \gc \textbf{75.21\imp{+0.58}}
& \gc \textbf{73.12\imp{+1.56}} \\

\bottomrule
\end{tabular}%
}}
\caption{\textbf{Severe imbalance with high effective classes.}
We report Top-1 accuracy across 10 datasets. Differences with base methods are shown as subscripts; Positive gains are highlighted in \color{Better}{green}.}
\label{tab:severe_high_eff}
\end{table*}

%% file: tables/severe_low_vit16_means.tex
\begin{table}[t]
\centering
\scriptsize
\setlength{\tabcolsep}{2pt}
\renewcommand{\arraystretch}{0.95}

\scalebox{1}[0.85]{%
\begin{tabular}{@{}llcccc@{}}
\toprule
& \multirow{2}{*}{\textbf{Method}}
& \multicolumn{2}{c}{\textbf{4-shot}}
& \multicolumn{2}{c}{\textbf{16-shot}} \\
\cmidrule(lr){3-4}
\cmidrule(l){5-6}
& & Base & w/\ours
  & Base & w/\ours \\
\midrule


\multirow{5}{*}{%
\rotatebox{90}{\tiny\shortstack{\textbf{Training}\\\textbf{Based}}}}

& CoOp
& 73.82 & \gc \textbf{79.75\imp{+5.93}}
& 74.72 & \gc \textbf{81.78\imp{+7.06}} \\

& LP++$_{\textbf{(CE)}}$
& 82.21 & \gc \textbf{85.45\imp{+3.24}}
& 81.28 & \gc \textbf{86.42\imp{+5.14}} \\

& LP++$_{\textbf{(FL)}}$
& 83.48 & \gc \textbf{85.20\imp{+1.72}}
& 82.84 & \gc \textbf{86.38\imp{+3.55}} \\

& ProKeR + CLAP
& 83.85 & \gc \textbf{85.81\imp{+1.96}}
& 82.73 & \gc \textbf{87.11\imp{+4.38}} \\

& ProLIP
& 86.05 & \gc \textbf{85.92\wor{-0.13}}
& 85.02 & \gc \textbf{86.69\imp{+1.67}} \\

\cmidrule(lr){2-6}

\multicolumn{6}{@{}l@{}}{\dashfill} \\


\multirow{4}{*}{%
\rotatebox{90}{\tiny\shortstack{\textbf{Training}\\\textbf{Free}}}}

& Tip-Adapter
& 74.14 & \gc \textbf{86.52\imp{+12.38}}
& 55.29 & \gc \textbf{87.03\imp{+31.75}} \\

& APE
& 82.91 & \gc \textbf{86.28\imp{+3.36}}
& 79.44 & \gc \textbf{86.66\imp{+7.22}} \\

& GDA
& 75.93 & \gc \textbf{78.63\imp{+2.69}}
& 77.40 & \gc \textbf{78.01\imp{+0.61}} \\

& ProKeR
& 84.64 & \gc \textbf{86.30\imp{+1.66}}
& 83.74 & \gc \textbf{87.32\imp{+3.59}} \\

\bottomrule
\end{tabular}%
}

\caption{\textbf{Severe imbalance with low effective classes.}
Mean Top-1 accuracy across 10 datasets.
Differences with base methods are shown as subscripts;
Positive gains are highlighted in \color{Better}{green}.}

\label{tab:severe_low_eff_means}
\end{table}

%% file: tables/bal_low_high_vit16_means.tex
\begin{table}[t]
\centering
\scriptsize
\setlength{\tabcolsep}{2pt}
\renewcommand{\arraystretch}{0.95}

\scalebox{1}[0.85]{%
\begin{tabular}{@{}llcccc@{}}
\toprule
& \multirow{2}{*}{\textbf{Method}}
& \multicolumn{2}{c}{\textbf{Low effective classes}}
& \multicolumn{2}{c}{\textbf{High effective classes}} \\
\cmidrule(lr){3-4}
\cmidrule(l){5-6}
& & \textbf{4-shot} & \textbf{16-shot}
  & \textbf{4-shot} & \textbf{16-shot} \\
\midrule


\multirow{10}{*}{%
\rotatebox{90}{\tiny\shortstack{\textbf{Training} \textbf{Based}}}}

& CoOp
& 80.18 & 84.49
& 65.28 & 72.49 \\
& \gc w/\ours
& \gc \textbf{84.23\imp{+4.05}}
& \gc \textbf{87.85\imp{+3.36}}
& \gc \textbf{69.07\imp{+3.80}}
& \gc \textbf{75.27\imp{+2.78}} \\

& LP++$_{\textbf{(CE)}}$
& 85.91 & 88.17
& 73.32 & 76.06 \\
& \gc w/\ours
& \gc \textbf{87.29\imp{+1.39}}
& \gc \textbf{89.53\imp{+1.36}}
& \gc \textbf{73.72\imp{+0.41}}
& \gc \textbf{76.42\imp{+0.35}} \\

& LP++$_{\textbf{(FL)}}$
& 86.42 & 88.31
& 73.69 & 76.51 \\
& \gc w/\ours
& \gc \textbf{86.64\imp{+0.22}}
& \gc \textbf{88.55\imp{+0.24}}
& \gc \textbf{72.84\wor{-0.85}}
& \gc \textbf{75.40\wor{-1.11}} \\

& ProKeR + CLAP
& 87.25 & 89.23
& 74.86 & 77.76 \\
& \gc w/\ours
& \gc \textbf{88.07\imp{+0.81}}
& \gc \textbf{90.86\imp{+1.63}}
& \gc \textbf{75.86\imp{+1.00}}
& \gc \textbf{79.47\imp{+1.70}} \\

& ProLIP
& 88.57 & 89.97
& 76.43 & 79.19 \\
& \gc w/\ours
& \gc \textbf{88.33\wor{-0.24}}
& \gc \textbf{90.69\imp{+0.72}}
& \gc \textbf{76.23\wor{-0.20}}
& \gc \textbf{80.31\imp{+1.12}} \\

\cmidrule(lr){2-6}


\multirow{8}{*}{%
\rotatebox{90}{\tiny\shortstack{\textbf{Training} \textbf{Free}}}}

& Tip-Adapter
& 81.68 & 68.32
& 62.98 & 45.98 \\
& \gc w/\ours
& \gc \textbf{88.02\imp{+6.34}}
& \gc \textbf{88.81\imp{+20.50}}
& \gc \textbf{74.90\imp{+11.92}}
& \gc \textbf{75.77\imp{+29.79}} \\

& APE
& 84.71 & 83.37
& 68.93 & 63.22 \\
& \gc w/\ours
& \gc \textbf{88.03\imp{+3.31}}
& \gc \textbf{89.46\imp{+6.09}}
& \gc \textbf{74.69\imp{+5.76}}
& \gc \textbf{74.21\imp{+10.99}} \\

& GDA
& 79.79 & 83.82
& 61.40 & 67.38 \\
& \gc w/\ours
& \gc \textbf{81.95\imp{+2.16}}
& \gc \textbf{85.83\imp{+2.01}}
& \gc \textbf{65.27\imp{+3.87}}
& \gc \textbf{72.34\imp{+4.96}} \\

& ProKeR
& 87.44 & 89.31
& 75.36 & 78.65 \\
& \gc w/\ours
& \gc \textbf{88.38\imp{+0.94}}
& \gc \textbf{90.49\imp{+1.18}}
& \gc \textbf{75.94\imp{+0.58}}
& \gc \textbf{79.06\imp{+0.41}} \\

\bottomrule
\end{tabular}%
}

\caption{\textbf{Near balanced with low and high effective classes.}
We report mean Top-1 accuracy across 10 datasets.
Differences with base methods are shown as subscripts; Positive gains are highlighted in \color{Better}{green}.}
\label{tab:balanced_low_high_eff_mean}
\end{table}

%% file: plots/lr_ablation.tex
\begin{figure}[t]
    \centering
    \includegraphics[width=0.95\columnwidth]{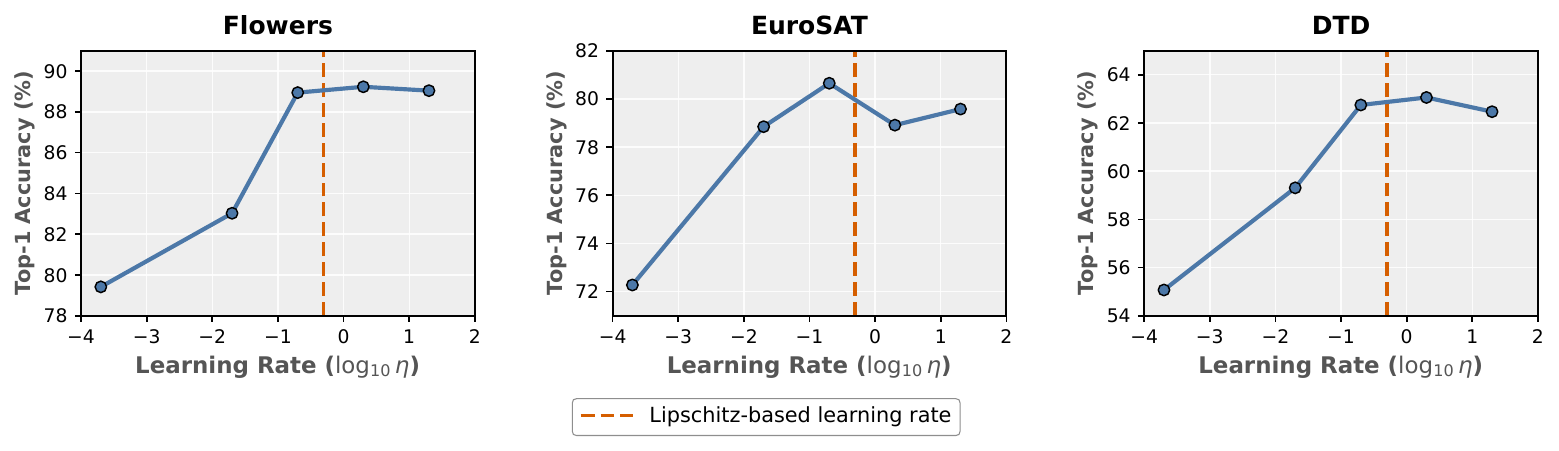}
    \caption{\textbf{Sensitivity of \ours{} to prototype-block learning rate.}
    Top-1 accuracy on Flowers, EuroSAT, and DTD for \ours{} applied on top of ProKeR in the 16-shot severe imbalance setting, with learning rates ranging from $2 \times 10^{-4}$ to $20$.}
    \label{fig:lr}
\end{figure}

%% file: plots/prism_ablations.tex
\begin{figure}[t]
    \centering

    \begin{subfigure}[b]{0.55\columnwidth}
        \centering
        \includegraphics[
            width=\linewidth,
            keepaspectratio
        ]{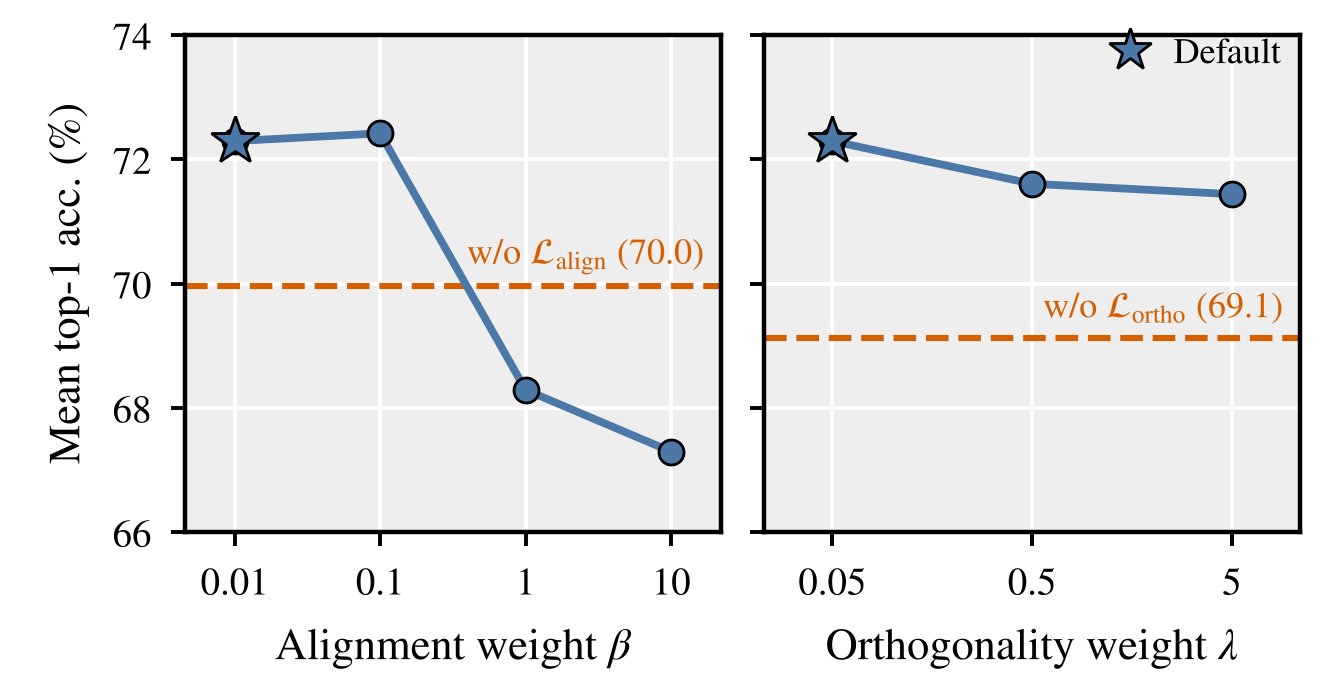}
        \caption{Regularization sensitivity.}
        \label{fig:regularization_ablation}
    \end{subfigure}
    \hfill
    \begin{subfigure}[b]{0.40\columnwidth}
        \centering
        \includegraphics[
            width=\linewidth,
            keepaspectratio
        ]{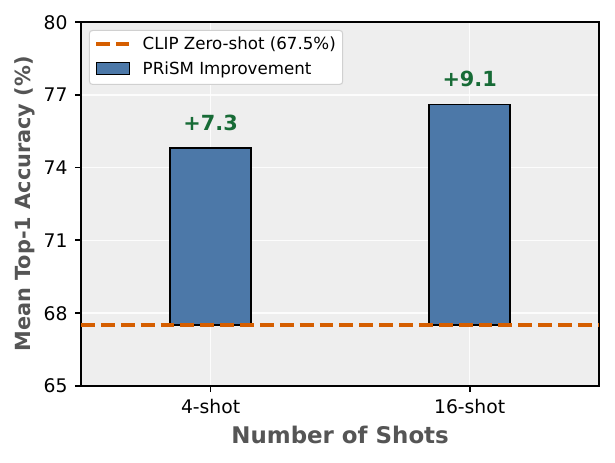}
        \caption{\ours{} on top of CLIP.}
        \label{fig:on_top_clip_ablation}
    \end{subfigure}

    \caption{\textbf{Ablation studies of \ours{}.}
    (a) Sensitivity to alignment ($\beta$) and orthogonality
    ($\lambda$) regularization weights.
    (b) \ours{} on top of CLIP.
    }

    \label{fig:combined_ablation}
\end{figure}

%% file: sec/5-conclusion.tex
\section{Conclusion}

We revisited few‑shot adaptation for VLMs by relaxing unrealistic assumptions on class balance and task composition. Using a Dirichlet-based benchmark, we show that many state-of-the-art methods degrade under class imbalance and large numbers of classes, even with increased supervision. To address this, we propose \ours, a plug‑and‑play class‑prototype regularization that consistently improves performance across diverse settings. \ours{} is optimized via an efficient block majorize–minimize algorithm with theoretically grounded Lipschitz-gradient bound, ensuring both effectiveness and scalability. Extensive experiments demonstrate that \ours{} substantially enhances robustness in challenging few‑shot regimes. 

%% file: sec/X_suppl.tex
\clearpage
\appendix
\setcounter{page}{1}
\maketitlesupplementary
\section{Derivation for Equation \ref{Eq: quadratic form of our loss}}
For fixed  vision-encoder variables $\phi$, the objective in \ref{Eq: overall loss} reduces to \begin{equation}
\resizebox{1\columnwidth}{!}{$\displaystyle
    \mathcal{L}_w(\mathbf{w}) = \sum_c \|\boldsymbol{w}_c  - \boldsymbol{a}_c\|_2^2 + \beta \sum_{c} 
   \|\boldsymbol{w}_c  - \boldsymbol{\mu}_c(\boldsymbol{\phi})\|_2^2  \\ - \frac{\lambda}{2(C-1)}\sum_{c\#c'} \|\boldsymbol{w}_c - \boldsymbol{w}_{c'}\|_2^2 
    $}
\end{equation}
Expanding each squared term,  we have
\begin{equation}
\begin{aligned}
\mathcal{L}_{\mathbf{w}}(\mathbf{w})
\overset{c}{=}&
(1+\beta)\sum_c\|\mathbf{w}_c\|_2^2\\
&-2\sum_c
\left(\mathbf{a}_c+
\beta\boldsymbol{\mu}_c(\phi)\right)^\top
\mathbf{w}_c\\
&-\frac{\lambda}{2(C-1)}
\sum_{c\neq c'}
\Big(
\|\mathbf{w}_c\|_2^2+
\|\mathbf{w}_{c'}\|_2^2
-2\mathbf{w}_c^\top\mathbf{w}_{c'}
\Big)
\end{aligned}
\end{equation}
where $\overset{c}{=}$ means equality up to constant. Since
\begin{equation}
\sum_{c\neq c'}
\left(\|\mathbf{w}_c\|_2^2+
\|\mathbf{w}_{c'}\|_2^2\right)
=2(C-1)\sum_c\|\mathbf{w}_c\|_2^2
\end{equation} the objective simplifies to
\begin{equation}
\begin{aligned}
\mathcal{L}_{\mathbf{w}}(\mathbf{w})
\overset{c}{=}&
(1+\beta-\lambda)
\sum_c\|\mathbf{w}_c\|_2^2\\
&+\frac{\lambda}{C-1}
\sum_{c\neq c'}\mathbf{w}_c^\top\mathbf{w}_{c'}\\
&-2\sum_c\mathbf{b}_c^\top\mathbf{w}_c
\end{aligned}
\end{equation}
where
$\mathbf{b}_c=\mathbf{a}_c+
\beta\boldsymbol{\mu}_c(\phi)$. We define the matrix $\mathbf{R}\in\mathbb{R}^{C\times C}$ as
\begin{equation} 
\mathbf{R}_{cc'}=
\begin{cases}
1+\beta-\lambda, & c=c'\\[2pt]
\dfrac{\lambda}{C-1}, & c\neq c'
\end{cases}
\end{equation}
The objective can then be written as
\begin{equation}
\mathcal{L}_{\mathbf{w}}(\mathbf{w})
\overset{c}{=}
\sum_{c,c'}\mathbf{R}_{cc'}
\mathbf{w}_{c'}^\top\mathbf{w}_c
-2\sum_c\mathbf{b}_c^\top\mathbf{w}_c = \mathbf{w}^\top\Phi\mathbf{w}
-2\mathbf{B}^\top\mathbf{w}
\end{equation}
with $\mathbf{w}=[\mathbf{w}_1^\top,\ldots,
\mathbf{w}_C^\top]^\top$, $\mathbf{B} = [\mathbf{b}_1,...,\mathbf{b}_C] \in \mathbb{R}^{dC \times dC}$, and matrix $\mathbf{\Phi} = \mathbf{R} \otimes \mathbf{I}_d$, where $\otimes$ denotes the Kronecker product and $\mathbf{I}_d$ is the $d \times d$ identity matrix.

\section{Derivation for Equation \ref{Eq: prototypes_majorizer}}

Since $\mathcal{L}_w(\mathbf{w})$ in Eq.~(3) is a quadratic, twice-differentiable function,  it has a Lipschitz continuous gradient $\rho>0$ satisfying
\begin{equation}
\|\nabla^2\mathcal{L}_w(\mathbf{w})\|_2
=2\|\Phi\|_2\leq\rho
\end{equation}
Using the second-order Taylor expansion around
$\mathbf{w}^j$, we have
\begin{equation}
\begin{aligned}
\mathcal{L}_w(\mathbf{w})
={}&\mathcal{L}_w(\mathbf{w}^j)
+\nabla\mathcal{L}_w(\mathbf{w}^j)^\top
(\mathbf{w}-\mathbf{w}^j)\\
&+(\mathbf{w}-\mathbf{w}^j)^\top
\Phi(\mathbf{w}-\mathbf{w}^j)
\end{aligned}
\end{equation}
Since
\begin{equation}
\begin{aligned}
&(\mathbf{w}-\mathbf{w}^j)^\top
\Phi(\mathbf{w}-\mathbf{w}^j)\\
&\quad\leq
\|\Phi\|_2\|\mathbf{w}-\mathbf{w}^j\|_2^2
\leq\frac{\rho}{2}
\|\mathbf{w}-\mathbf{w}^j\|_2^2
\end{aligned}
\end{equation}
we obtain the majorizing function
\begin{equation}
\resizebox{\columnwidth}{!}{$\displaystyle
\mathcal{L}_w(\mathbf{w})
\leq \mathcal{M}_w(\mathbf{w},\mathbf{w}^j)
=\mathcal{L}_w(\mathbf{w}^j)
+\nabla\mathcal{L}_w(\mathbf{w}^j)^\top
(\mathbf{w}-\mathbf{w}^j)
+\frac{\rho}{2}
\|\mathbf{w}-\mathbf{w}^j\|_2^2
$}
\end{equation}

\section{Proof of Proposition \ref{Eq: Lp constant prototypes}}

The spectral norm of the Hessian of quadratic objective ${\cal L}_w$ is given by:
\[2 \|\mathbf{\Phi}\|_2 = 2 \| \mathbf{R}\otimes \mathbf{I}_d \|_2 = 2\| \mathbf{R} \|_2  \| \mathbf{I}_d \|_2 = 2 \| \mathbf{R} \|_2\]
The second equality above comes from the spectral-norm property of the Kronecker product, and the last comes from the fact $\| \mathbf{I}_d \|_2 = 1$. 
Now we need to find an upper bound on $\| \mathbf{R}\|_2$. As $\mathbf{R}$ is symmetric, $\| \mathbf{R}\|_2$ is the largest absolute eigenvalue of $\mathbf{R}$. 
By Gershgorin’s circle theorem, for each eigenvalue $s$ of $\mathbf{R}$, there exists at least a row of the matrix such that $s$ lies within a disk centered at the diagonal entry of the row, 
with a radius equal to the sum of absolute values of the off-diagonal elements in that row. 
Now observe that all the diagonal elements of $\mathbf{R}$ are given by
\begin{equation}
\label{appendix: diagonal-elements}
\mathbf{R}_{cc} = \beta + 1 - \lambda
\end{equation}
 Also, for all rows of $\mathbf{R}$, the sum of absolute values of the non-diagonal elements is given by:
 \begin{equation}
 \label{appendix: off-diagonal-elements}
    \sum\limits_{c'\#c} |\ \mathbf{R}_{c'c} |\ = (C-1) \frac{\lambda}{C-1} = \lambda
\end{equation}
From \eqref{appendix: diagonal-elements} and \eqref{appendix: off-diagonal-elements}, it follows that every eigenvalue $s$ of $\mathbf{R}$ satisfies: 
\begin{equation}
    | s | \leq |\ \beta + 1 - \lambda |\ + \lambda
\end{equation}
Therefore: 
\begin{equation}
    \| \mathbf{R} \|_2 \leq |\ \beta + 1 - \lambda |\ + \lambda
\end{equation}
As $\|\mathbf{\Phi}\|_2 = \| \mathbf{R} \|_2$, we obtain the following upper bound on 
the spectral norm of the Hessian of quadratic objective ${\cal L}_w$, which provides 
a valid Lipschitz constant: 
\begin{equation}
   \rho = 2 (|\ \beta + 1 - \lambda |\ + \lambda)  \geq 2 \| \mathbf{R} \|_2 
\end{equation}
\section{Effect of the Shot Budget on Class Imbalance}

Let $\pi_S$ be a fixed nonuniform Dirichlet draw over
$C_{\mathrm{eff}}$ classes, and let
$\pi_U=\frac{1}{C_{\mathrm{eff}}}\mathbf{1}$ denote the uniform distribution. Following our benchmark design, given a $K$-shot budget, we first assign one sample per class to ensure a valid few-shot adaptation task, then distribute the remaining $(K-1)C_{\mathrm{eff}}$ samples according to $\pi_S$. The resulting class proportions in the support set are given by
\begin{equation}
\begin{aligned}
\tilde{\pi}_{S,c}^{(K)}
&=\frac{1+(K-1)C_{\mathrm{eff}}\pi_{S,c}}
{KC_{\mathrm{eff}}}\\
&=\frac{1}{K}\pi_{U,c}
+\left(1-\frac{1}{K}\right)\pi_{S,c}
\end{aligned}
\end{equation}
Therefore,
\begin{equation}
\left\|\tilde{\pi}_S^{(K)}-\pi_U\right\|_2^2
=
\left(1-\frac{1}{K}\right)^2
\left\|\pi_S-\pi_U\right\|_2^2
\end{equation}
Since $\pi_S\neq\pi_U$, we have
$\|\pi_S-\pi_U\|_2^2>0$.
Moreover, $(1-1/K)^2$ is strictly increasing for
$K\geq 1$. Thus, for a fixed nonuniform Dirichlet
draw, the deviation of the support class proportions from uniformity increases strictly with the shot budget $K$. To further illustrate this effect, Fig.~\ref{fig:shots_imbalance}
presents the average majority-minority count gap,
defined as the difference between the most and least represented classes in the support set. We evaluate $400$ few-shot tasks on EuroSAT, fixing the effective number of classes to $80\%$ and the Dirichlet concentration parameter to $\delta=0.2$. As we can see, this gap  increases from $5$ at 2-shot to $14.5$ at 4 shot, and $71$ at 16-shot, highlighting the growing disparity in support samples counts as shot budget increases.   

\begin{figure}[t]
    \centering
    \includegraphics[width=0.55\columnwidth]{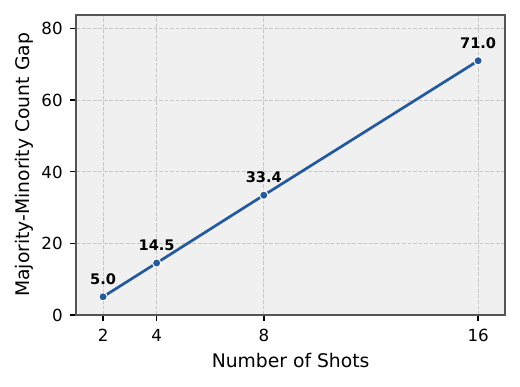}
    \caption{
    \textbf{Effect of the Shot Budget on Class Imbalance}
    }
    \label{fig:shots_imbalance}
    \vspace{-3mm}
\end{figure}
\section{Effect of Dirichlet Concentration Parameter on Class Imbalance}
\label{app:dirichlet}
\begin{figure}[ht]
    \centering
    \begin{minipage}[b]{0.4\columnwidth}
        \centering
        \includegraphics[width=\linewidth]{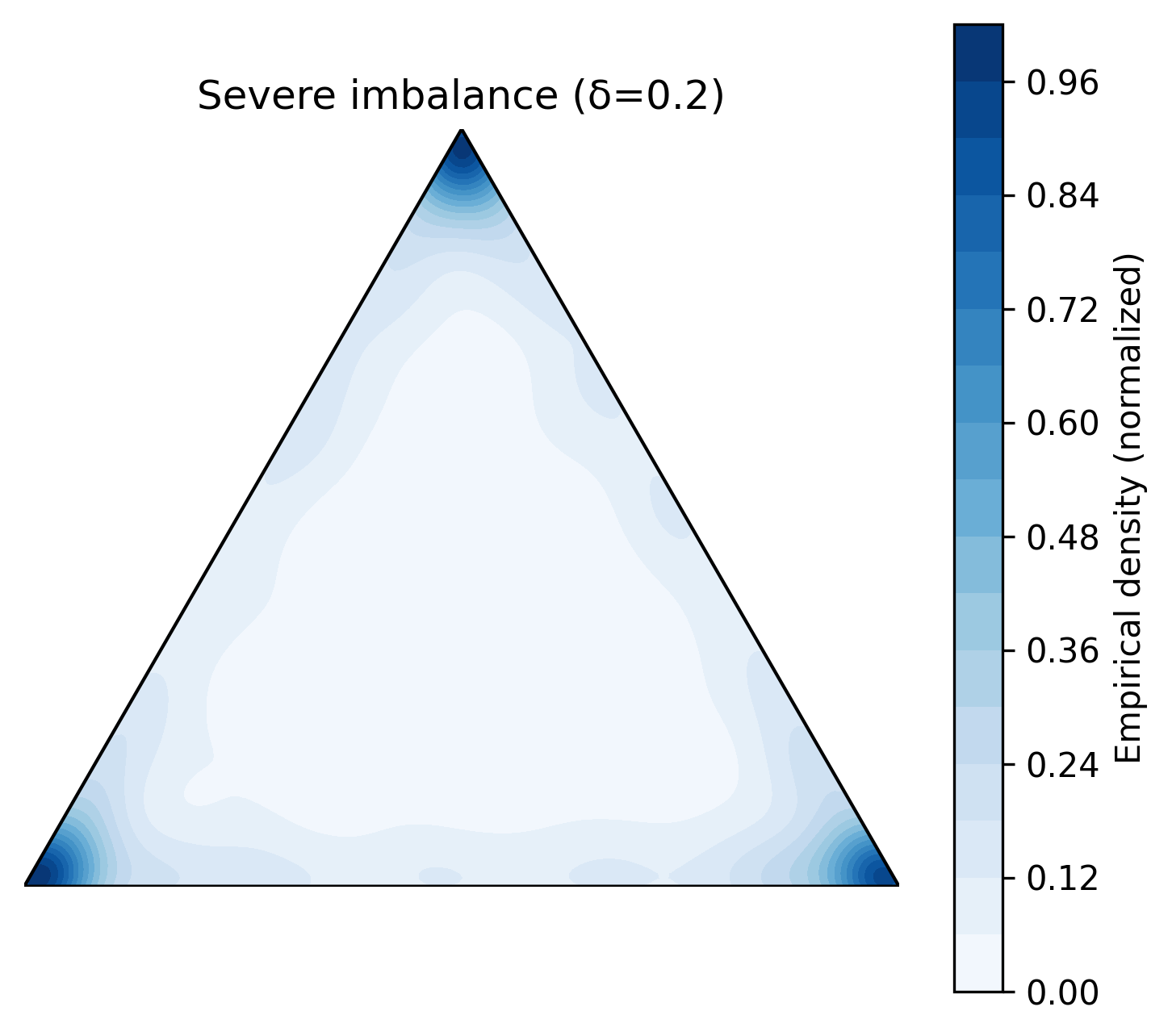}
    \end{minipage}\hfill%
    \begin{minipage}[b]{0.4\columnwidth}
        \centering
        \includegraphics[width=\linewidth]{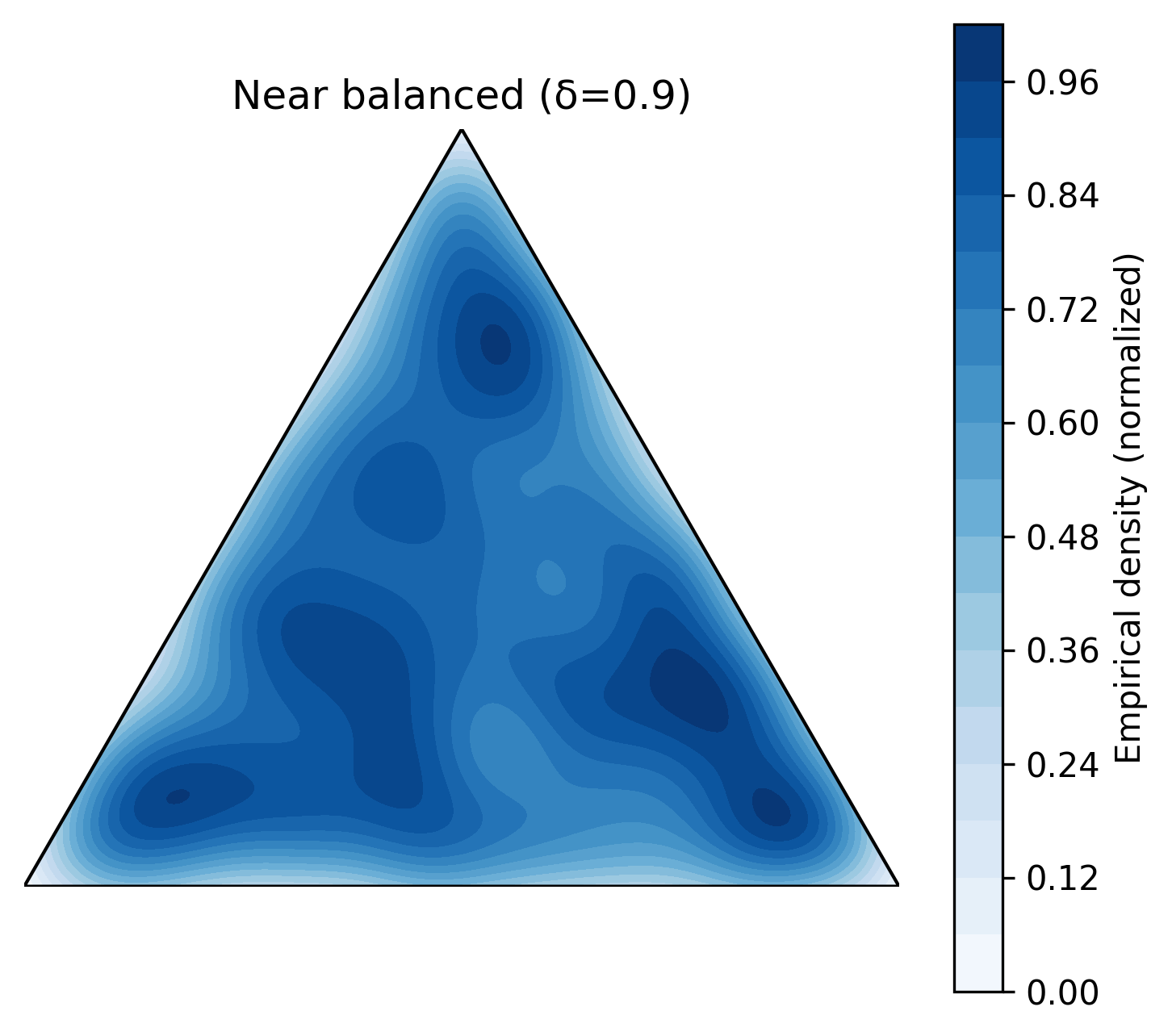}
    \end{minipage}
    \caption{Effect of the Dirichlet concentration parameter
    $\delta$ on class imbalance across three classes.
    Smaller $\delta$ values favor more skewed class proportions,
    while larger values concentrate proportions closer to uniform.}
    \label{fig:dirichlet}
\end{figure}
\noindent Figure \ref{fig:dirichlet} illustrates the effect of the Dirichlet concentration parameter $\delta$ on class imbalance across three classes. Smaller values of $\delta$ produce more skewed class distributions, whereas larger values lead to more uniform sampling.
\section{Block Coordinate Majorize-Minimize Algorithm}
\input{algorithms/algorithm1}

\section{Further Implementation Details}
For all baselines, we follow the implementation protocols described in their original papers, including the recommended hyperparameters and training settings. For \ours{}, we select hyperparameters using the ImageNet validation set and then apply them across all downstream datasets without further tuning. Code will be made publicly available.
\section{Generalization Across Vision Backbones}
\input{tables_supp/severe_high_vitl14_detailed}
Here we investigate the behavior of the proposed benchmark and the generalization of \ours{} to larger vision backbones, such as CLIP ViT-L/14. In table \ref{tab:severe_high_vitl14_detailed} we present results using this backbone across $10$ datasets under severe class imbalance with a high number of effective classes. For consistency, we use the same hyperparameters for \ours{} as those employed with ViT-B/16. Reported results show that, even with a larger backbone, several baseline few-shot adaption methods remain vulnerable to realistic class distributions, as reflected in the performance degradation observed when increasing the number of shots from $4$ to $16$. In particular, Proker exhibits a decrease in mean accuracy from $78.47\%$ to $77.81\%$ while LP++ with focal loss drops from $75.88\%$ to $74.54\%$. Although this regime is more challenging, \ours{} consistently improves performance across different adaptation strategies. For instance APE gains $\boldsymbol{+5.16}$ at 16 shots while ProKeR, CoOp, and LP++ with focal loss improve by $\boldsymbol{+1.54}$, $\boldsymbol{+2.61}$, and $\boldsymbol{+3.65}$, respectively.      
\clearpage
\onecolumn
\section{Detailed Results for Severe Imbalanced Setting with Low Effective Number of Classes}
\input{tables_supp/severe_low_vit16_detailed}
\clearpage
\section{Detailed Results for Near Balanced Setting with High Effective Number of Classes}
\input{tables_supp/bal_high_vit16_detailed}
\clearpage
\section{Detailed Results for Near Balanced Setting with Low Effective Number of Classes}
\input{tables_supp/bal_low_vit16_detailed}
\twocolumn

%% file: algorithms/algorithm1.tex
\begin{algorithm}[!htbp]
\RestyleAlgo{ruled}
$\beta = 0.01$;
$\delta = 0.05$; \\
Initialize ${\mathbf w}^{0,0} = \mathbf{a}$ {\color{blue}\tcp{\small Initial prototypes from the base method}}
Initialize $\boldsymbol{\phi}^{0,0}$  {\color{blue}\tcp{\small Pretrained vision encoder with LoRA parameters}} 
\For{$j = 0, 1, \ldots$}
{ \For{$l_{1} = 0, 1, \ldots,$ {\em iter}$_{\mathbf{w}}$}
{
$\mathbf{w}^{j,l_{1}+1} = \mathbf{w}^{j,l_{1}} - \frac{1}{\rho} \nabla L_{\mathbf{w}}(\mathbf{w}^{j,l_{1}},\boldsymbol{\phi}^{j,0})$ \\ {\color{blue}\tcp{\small $1/\rho$, from Eq.(\ref{Eq: Lp constant prototypes})}}} 
$\mathbf{w}^{j+1,0} = \mathbf{w}^{j,\mbox{{\small iter}}_{\mathbf{w}}}$
\\ 
\For{$l_{2} = 0, 1, \ldots,$ {\em iter}$_{\boldsymbol{\phi}}$}
{
$\boldsymbol{\phi}^{j,l_{2}+1} = \boldsymbol{\phi}^{j,l_{2}} - {\boldsymbol{\eta}_{\boldsymbol{\phi}}} \nabla L_{\boldsymbol{\phi}} (\mathbf{w}^{j+1,0},\boldsymbol{\phi}^{j,l_{2}})$\\
{\color{blue}\tcp{\small $\boldsymbol{\eta}_{\boldsymbol{\phi}}$, Standard learning rate for LoRA-based encoder adaptation}}
}$\boldsymbol{\phi}^{j+1,0} = \boldsymbol{\phi}^{j,\mbox{{\small iter}}_{\boldsymbol{\phi}}}$
\\ }
\caption{Block coordinate MM (${\mathbf w}$, $\boldsymbol{\phi}$) \label{Algorithm-MM-w-alpha}}
\end{algorithm}

%% file: tables_supp/severe_high_vitl14_detailed.tex
\begin{table*}[!htbp]
\centering
\small
\setlength{\tabcolsep}{1pt}
\renewcommand{\arraystretch}{1.08}
\resizebox{0.9\textwidth}{!}{%
\scalebox{1}[0.95]{%
\begin{tabular}{ll|cccccccccc|c}
\toprule
&& \textbf{Caltech} & \textbf{DTD} & \textbf{EuroSAT} & \textbf{FGVC} & \textbf{Food101} & \textbf{Flowers} & \textbf{Pets} & \textbf{Cars} & \textbf{Sun397} & \textbf{UCF101} & \textbf{Mean} \\
\midrule

\multicolumn{13}{c}{\textbf{16-shot}} \\
\midrule

\multirow{6}{*}{\rotatebox{90}{\shortstack{\textbf{Training Based}}}}
&CoOp & 87.44 & 47.78 & 63.65 & 38.00 & 71.28 & 87.10 & 82.41 & 69.42 & 54.22 & 66.84 & 66.81 \\
& \gc w/ \ours & \gc \textbf{91.64\imp{+4.21}} & \gc \textbf{54.86\imp{+7.08}} & \gc \textbf{65.20\imp{+1.55}} & \gc \textbf{37.09\wor{-0.92}} & \gc \textbf{72.65\imp{+1.36}} & \gc \textbf{88.36\imp{+1.26}} & \gc \textbf{81.71\wor{-0.69}} & \gc \textbf{66.20\wor{-3.22}} & \gc \textbf{62.27\imp{+8.05}} & \gc \textbf{74.28\imp{+7.44}} & \gc \textbf{69.43\imp{+2.61}} \\


&LP++$_\textbf{{(FL)}}$ & 94.89 & 54.15 & 70.54 & 41.37 & 84.56 & 89.54 & 90.72 & 75.83 & 66.72 & 77.06 & 74.54 \\
& \gc w/ \ours & \gc \textbf{95.66\imp{+0.77}} & \gc \textbf{62.18\imp{+8.04}} & \gc \textbf{73.83\imp{+3.30}} & \gc \textbf{42.55\imp{+1.17}} & \gc \textbf{88.34\imp{+3.77}} & \gc \textbf{92.35\imp{+2.81}} & \gc \textbf{92.90\imp{+2.18}} & \gc \textbf{78.27\imp{+2.44}} & \gc \textbf{74.49\imp{+7.77}} & \gc \textbf{81.35\imp{+4.29}} & \gc \textbf{78.19\imp{+3.65}} \\

&ProKeR + CLAP & 94.61 & 58.82 & 69.74 & 42.91 & 87.21 & 92.54 & 88.66 & 73.38 & 69.38 & 79.34 & 75.66 \\
& \gc w/ \ours & \gc \textbf{94.77\imp{+0.17}} & \gc \textbf{67.87\imp{+9.05}} & \gc \textbf{76.59\imp{+6.84}} & \gc \textbf{46.50\imp{+3.60}} & \gc \textbf{87.84\imp{+0.63}} & \gc \textbf{95.46\imp{+2.92}} & \gc \textbf{92.44\imp{+3.78}} & \gc \textbf{74.41\imp{+1.03}} & \gc \textbf{70.65\imp{+1.28}} & \gc \textbf{81.54\imp{+2.21}} & \gc \textbf{78.81\imp{+3.15}} \\

\multicolumn{13}{@{}l@{}}{\dashfill} \\

\multirow{6}{*}{\rotatebox{90}{\textbf{Training Free}}}

&APE & 94.95 & 58.56 & 25.79 & 24.01 & 83.22 & 80.94 & 88.95 & 61.00 & 70.42 & 77.34 & 66.52 \\
& \gc w/ \ours & \gc \textbf{95.08\imp{+0.13}} & \gc \textbf{62.76\imp{+4.20}} & \gc \textbf{56.46\imp{+30.67}} & \gc \textbf{28.68\imp{+4.68}} & \gc \textbf{84.30\imp{+1.08}} & \gc \textbf{88.85\imp{+7.91}} & \gc \textbf{87.18\wor{-1.78}} & \gc \textbf{61.30\imp{+0.29}} & \gc \textbf{72.18\imp{+1.76}} & \gc \textbf{79.95\imp{+2.61}} & \gc \textbf{71.67\imp{+5.16}} \\

&GDA & 92.64 & 53.56 & 67.27 & 39.40 & 73.95 & 89.83 & 80.94 & 68.53 & 65.55 & 75.48 & 70.72 \\
& \gc w/ \ours & \gc \textbf{92.87\imp{+0.23}} & \gc \textbf{54.54\imp{+0.98}} & \gc \textbf{68.59\imp{+1.32}} & \gc \textbf{40.61\imp{+1.21}} & \gc \textbf{76.51\imp{+2.56}} & \gc \textbf{90.74\imp{+0.91}} & \gc \textbf{81.86\imp{+0.92}} & \gc \textbf{69.31\imp{+0.78}} & \gc \textbf{68.29\imp{+2.74}} & \gc \textbf{76.13\imp{+0.65}} & \gc \textbf{71.95\imp{+1.23}} \\

&ProKeR & 94.87 & 63.73 & 73.29 & 45.01 & 89.21 & 92.88 & 89.76 & 76.62 & 72.19 & 80.58 & 77.81 \\
& \gc w/ \ours & \gc \textbf{95.03\imp{+0.16}} & \gc \textbf{67.22\imp{+3.49}} & \gc \textbf{78.23\imp{+4.93}} & \gc \textbf{47.27\imp{+2.27}} & \gc \textbf{88.36\wor{-0.85}} & \gc \textbf{95.70\imp{+2.82}} & \gc \textbf{92.23\imp{+2.47}} & \gc \textbf{75.86\wor{-0.77}} & \gc \textbf{71.58\wor{-0.61}} & \gc \textbf{82.09\imp{+1.52}} & \gc \textbf{79.36\imp{+1.54}} \\

\midrule

\multicolumn{13}{c}{\textbf{4-shot}} \\
\midrule

\multirow{6}{*}{\rotatebox{90}{\shortstack{\textbf{Training Based}}}}
&CoOp & 86.85 & 43.58 & 59.38 & 32.83 & 74.63 & 85.94 & 82.78 & 68.23 & 52.09 & 65.03 & 65.13 \\
& \gc w/ \ours & \gc \textbf{90.80\imp{+3.95}} & \gc \textbf{50.19\imp{+6.61}} & \gc \textbf{65.02\imp{+5.63}} & \gc \textbf{32.54\wor{-0.29}} & \gc \textbf{75.42\imp{+0.79}} & \gc \textbf{86.79\imp{+0.86}} & \gc \textbf{78.73\wor{-4.05}} & \gc \textbf{62.02\wor{-6.21}} & \gc \textbf{59.95\imp{+7.86}} & \gc \textbf{72.09\imp{+7.06}} & \gc \textbf{67.35\imp{+2.22}} \\


&LP++$_\textbf{{(FL)}}$ & 94.94 & 55.85 & 72.59 & 39.38 & 88.14 & 90.44 & 92.27 & 77.31 & 69.05 & 78.79 & 75.88 \\
& \gc w/ \ours & \gc \textbf{95.49\imp{+0.55}} & \gc \textbf{61.22\imp{+5.37}} & \gc \textbf{73.97\imp{+1.38}} & \gc \textbf{40.28\imp{+0.90}} & \gc \textbf{89.76\imp{+1.62}} & \gc \textbf{92.21\imp{+1.77}} & \gc \textbf{93.23\imp{+0.96}} & \gc \textbf{78.29\imp{+0.98}} & \gc \textbf{74.22\imp{+5.17}} & \gc \textbf{80.66\imp{+1.88}} & \gc \textbf{77.93\imp{+2.06}} \\

&ProKeR + CLAP & 94.57 & 58.01 & 71.58 & 39.99 & 89.28 & 93.48 & 90.03 & 75.37 & 71.46 & 80.24 & 76.40 \\
& \gc w/ \ours & \gc \textbf{94.19\wor{-0.38}} & \gc \textbf{62.73\imp{+4.71}} & \gc \textbf{72.83\imp{+1.25}} & \gc \textbf{42.69\imp{+2.70}} & \gc \textbf{87.28\wor{-2.00}} & \gc \textbf{95.17\imp{+1.69}} & \gc \textbf{92.25\imp{+2.22}} & \gc \textbf{74.18\wor{-1.19}} & \gc \textbf{69.68\wor{-1.78}} & \gc \textbf{80.76\imp{+0.52}} & \gc \textbf{77.18\imp{+0.78}} \\

\multicolumn{13}{@{}l@{}}{\dashfill} \\

\multirow{6}{*}{\rotatebox{90}{\textbf{Training Free}}}

&APE & 95.46 & 57.44 & 47.60 & 19.60 & 91.53 & 82.30 & 93.60 & 70.68 & 71.49 & 77.93 & 70.76 \\
& \gc w/ \ours & \gc \textbf{95.10\wor{-0.36}} & \gc \textbf{60.49\imp{+3.04}} & \gc \textbf{73.25\imp{+25.65}} & \gc \textbf{25.11\imp{+5.51}} & \gc \textbf{86.43\wor{-5.10}} & \gc \textbf{88.97\imp{+6.67}} & \gc \textbf{88.21\wor{-5.39}} & \gc \textbf{65.71\wor{-4.97}} & \gc \textbf{71.50\imp{+0.02}} & \gc \textbf{80.37\imp{+2.43}} & \gc \textbf{73.51\imp{+2.75}} \\

&GDA & 91.49 & 50.34 & 65.61 & 36.02 & 72.62 & 88.80 & 76.03 & 64.79 & 60.42 & 73.99 & 68.01 \\
& \gc w/ \ours & \gc \textbf{91.93\imp{+0.43}} & \gc \textbf{50.91\imp{+0.57}} & \gc \textbf{65.85\imp{+0.25}} & \gc \textbf{37.06\imp{+1.04}} & \gc \textbf{73.90\imp{+1.28}} & \gc \textbf{89.55\imp{+0.74}} & \gc \textbf{77.10\imp{+1.06}} & \gc \textbf{65.46\imp{+0.67}} & \gc \textbf{62.27\imp{+1.84}} & \gc \textbf{74.16\imp{+0.17}} & \gc \textbf{68.82\imp{+0.81}} \\

&ProKeR & 95.11 & 64.11 & 76.76 & 43.33 & 90.66 & 92.75 & 91.39 & 77.40 & 72.69 & 80.53 & 78.47 \\
& \gc w/ \ours & \gc \textbf{94.69\wor{-0.43}} & \gc \textbf{65.57\imp{+1.46}} & \gc \textbf{76.87\imp{+0.11}} & \gc \textbf{44.44\imp{+1.11}} & \gc \textbf{88.09\wor{-2.58}} & \gc \textbf{95.37\imp{+2.62}} & \gc \textbf{92.43\imp{+1.04}} & \gc \textbf{75.02\wor{-2.37}} & \gc \textbf{70.30\wor{-2.39}} & \gc \textbf{81.06\imp{+0.52}} & \gc \textbf{78.38\wor{-0.09}} \\

\bottomrule
\end{tabular}%
}}
\caption{\textbf{Generalization Across Vision Backbones.} We report Top-1 accuracy across 10 datasets under severe imbalance setting with high effective number of classes, using CLIP ViT-L/14 as the vision backbone. Differences with base methods are shown as subscripts; Positive gains are highlighted in \color{Better}{green}.}
\label{tab:severe_high_vitl14_detailed}
\end{table*}

%% file: tables_supp/severe_low_vit16_detailed.tex
\begin{table*}[!htbp]
\centering
\small
\setlength{\tabcolsep}{1pt}
\renewcommand{\arraystretch}{1.08}
\resizebox{0.9\textwidth}{!}{%
\scalebox{1}[0.95]{%
\begin{tabular}{ll|cccccccccc|c}
\toprule
&& \textbf{Caltech} & \textbf{DTD} & \textbf{EuroSAT} & \textbf{FGVC} & \textbf{Food101} & \textbf{Flowers} & \textbf{Pets} & \textbf{Cars} & \textbf{Sun397} & \textbf{UCF101} & \textbf{Mean} \\
\midrule

\multicolumn{13}{c}{\textbf{16-shot}} \\
\midrule

\multirow{10}{*}{\rotatebox{90}{\shortstack{\textbf{Training Based}}}}
&CoOp & 91.25 & 59.68 & 81.97 & 46.12 & 77.23 & 86.93 & 88.58 & 73.94 & 67.66 & 73.86 & 74.72 \\
& \gc w/ \ours & \gc \textbf{95.18\imp{+3.93}} & \gc \textbf{71.87\imp{+12.19}} & \gc \textbf{88.09\imp{+6.13}} & \gc \textbf{51.05\imp{+4.93}} & \gc \textbf{86.50\imp{+9.27}} & \gc \textbf{93.41\imp{+6.48}} & \gc \textbf{93.85\imp{+5.27}} & \gc \textbf{80.04\imp{+6.09}} & \gc \textbf{76.18\imp{+8.52}} & \gc \textbf{81.69\imp{+7.83}} & \gc \textbf{81.78\imp{+7.06}} \\

&LP++$_\textbf{{(CE)}}$ & 96.86 & 62.45 & 85.11 & 49.05 & 88.58 & 92.06 & 93.26 & 82.72 & 79.43 & 83.30 & 81.28 \\
& \gc w/ \ours & \gc \textbf{97.98\imp{+1.13}} & \gc \textbf{77.61\imp{+15.16}} & \gc \textbf{91.74\imp{+6.63}} & \gc \textbf{55.59\imp{+6.54}} & \gc \textbf{92.67\imp{+4.09}} & \gc \textbf{93.79\imp{+1.73}} & \gc \textbf{95.96\imp{+2.70}} & \gc \textbf{85.47\imp{+2.76}} & \gc \textbf{84.77\imp{+5.34}} & \gc \textbf{88.62\imp{+5.32}} & \gc \textbf{86.42\imp{+5.14}} \\

&LP++$_\textbf{{(FL)}}$ & 97.20 & 66.23 & 86.04 & 52.79 & 89.96 & 91.49 & 94.61 & 84.58 & 81.32 & 84.13 & 82.84 \\
& \gc w/ \ours & \gc \textbf{97.91\imp{+0.72}} & \gc \textbf{78.84\imp{+12.62}} & \gc \textbf{92.81\imp{+6.77}} & \gc \textbf{55.95\imp{+3.16}} & \gc \textbf{92.22\imp{+2.26}} & \gc \textbf{91.65\imp{+0.16}} & \gc \textbf{95.78\imp{+1.17}} & \gc \textbf{84.91\imp{+0.33}} & \gc \textbf{85.38\imp{+4.06}} & \gc \textbf{88.39\imp{+4.26}} & \gc \textbf{86.38\imp{+3.55}} \\

&ProKeR + CLAP & 96.84 & 64.28 & 89.09 & 51.44 & 90.03 & 93.68 & 91.44 & 83.40 & 81.58 & 85.53 & 82.73 \\
& \gc w/ \ours & \gc \textbf{97.52\imp{+0.67}} & \gc \textbf{77.61\imp{+13.33}} & \gc \textbf{94.50\imp{+5.41}} & \gc \textbf{57.85\imp{+6.41}} & \gc \textbf{91.76\imp{+1.73}} & \gc \textbf{96.25\imp{+2.57}} & \gc \textbf{96.64\imp{+5.20}} & \gc \textbf{86.30\imp{+2.91}} & \gc \textbf{83.73\imp{+2.15}} & \gc \textbf{88.95\imp{+3.42}} & \gc \textbf{87.11\imp{+4.38}} \\

&ProLIP & 97.87 & 69.12 & 88.06 & 54.15 & 91.94 & 94.87 & 94.60 & 87.01 & 84.99 & 87.58 & 85.02 \\
& \gc w/ \ours & \gc \textbf{97.94\imp{+0.07}} & \gc \textbf{76.09\imp{+6.97}} & \gc \textbf{90.34\imp{+2.28}} & \gc \textbf{56.23\imp{+2.08}} & \gc \textbf{91.32\wor{-0.61}} & \gc \textbf{97.01\imp{+2.14}} & \gc \textbf{97.36\imp{+2.76}} & \gc \textbf{87.16\imp{+0.15}} & \gc \textbf{84.06\wor{-0.93}} & \gc \textbf{89.40\imp{+1.83}} & \gc \textbf{86.69\imp{+1.67}} \\

\multicolumn{13}{@{}l@{}}{\dashfill} \\

\multirow{8}{*}{\rotatebox{90}{\textbf{Training Free}}}
&Tip-Adapter & 93.77 & 34.64 & 62.23 & 31.65 & 48.66 & 32.91 & 47.11 & 62.39 & 72.76 & 66.72 & 55.29 \\
& \gc w/ \ours & \gc \textbf{98.15\imp{+4.38}} & \gc \textbf{79.01\imp{+44.36}} & \gc \textbf{93.55\imp{+31.31}} & \gc \textbf{55.29\imp{+23.64}} & \gc \textbf{92.08\imp{+43.42}} & \gc \textbf{93.81\imp{+60.90}} & \gc \textbf{96.55\imp{+49.44}} & \gc \textbf{87.28\imp{+24.89}} & \gc \textbf{85.61\imp{+12.85}} & \gc \textbf{88.98\imp{+22.26}} & \gc \textbf{87.03\imp{+31.75}} \\

&APE & 97.85 & 69.35 & 79.98 & 41.27 & 90.60 & 74.27 & 93.48 & 78.44 & 83.82 & 85.31 & 79.44 \\
& \gc w/ \ours & \gc \textbf{98.19\imp{+0.34}} & \gc \textbf{78.44\imp{+9.09}} & \gc \textbf{94.18\imp{+14.20}} & \gc \textbf{53.83\imp{+12.56}} & \gc \textbf{92.09\imp{+1.49}} & \gc \textbf{94.33\imp{+20.06}} & \gc \textbf{96.79\imp{+3.30}} & \gc \textbf{84.53\imp{+6.09}} & \gc \textbf{85.40\imp{+1.58}} & \gc \textbf{88.78\imp{+3.47}} & \gc \textbf{86.66\imp{+7.22}} \\

&GDA & 95.38 & 65.95 & 86.62 & 47.51 & 74.50 & 84.59 & 87.17 & 73.98 & 77.05 & 81.26 & 77.40 \\
& \gc w/ \ours & \gc \textbf{95.46\imp{+0.08}} & \gc \textbf{66.62\imp{+0.67}} & \gc \textbf{86.78\imp{+0.16}} & \gc \textbf{48.42\imp{+0.90}} & \gc \textbf{75.87\imp{+1.37}} & \gc \textbf{84.98\imp{+0.39}} & \gc \textbf{87.35\imp{+0.18}} & \gc \textbf{75.00\imp{+1.02}} & \gc \textbf{78.03\imp{+0.99}} & \gc \textbf{81.61\imp{+0.36}} & \gc \textbf{78.01\imp{+0.61}} \\

&ProKeR & 97.32 & 69.23 & 84.86 & 53.60 & 91.24 & 93.19 & 93.31 & 84.85 & 83.60 & 86.19 & 83.74 \\
& \gc w/ \ours & \gc \textbf{97.89\imp{+0.58}} & \gc \textbf{79.35\imp{+10.12}} & \gc \textbf{93.39\imp{+8.53}} & \gc \textbf{58.12\imp{+4.52}} & \gc \textbf{92.03\imp{+0.79}} & \gc \textbf{95.92\imp{+2.74}} & \gc \textbf{96.77\imp{+3.46}} & \gc \textbf{86.45\imp{+1.60}} & \gc \textbf{84.35\imp{+0.75}} & \gc \textbf{88.98\imp{+2.79}} & \gc \textbf{87.32\imp{+3.59}} \\

\midrule

\multicolumn{13}{c}{\textbf{4-shot}} \\
\midrule

\multirow{10}{*}{\rotatebox{90}{\shortstack{\textbf{Training Based}}}}
&CoOp & 91.10 & 56.53 & 83.28 & 38.61 & 80.90 & 86.00 & 88.12 & 73.84 & 66.13 & 73.67 & 73.82 \\
& \gc w/ \ours & \gc \textbf{95.15\imp{+4.05}} & \gc \textbf{67.71\imp{+11.19}} & \gc \textbf{86.84\imp{+3.56}} & \gc \textbf{44.12\imp{+5.51}} & \gc \textbf{85.19\imp{+4.29}} & \gc \textbf{92.52\imp{+6.52}} & \gc \textbf{91.87\imp{+3.75}} & \gc \textbf{79.23\imp{+5.39}} & \gc \textbf{73.73\imp{+7.60}} & \gc \textbf{81.17\imp{+7.50}} & \gc \textbf{79.75\imp{+5.93}} \\

&LP++$_\textbf{{(CE)}}$ & 96.88 & 64.04 & 86.63 & 48.29 & 90.67 & 92.23 & 93.68 & 83.15 & 82.06 & 84.53 & 82.21 \\
& \gc w/ \ours & \gc \textbf{97.72\imp{+0.84}} & \gc \textbf{75.40\imp{+11.36}} & \gc \textbf{90.63\imp{+4.00}} & \gc \textbf{52.89\imp{+4.60}} & \gc \textbf{91.75\imp{+1.08}} & \gc \textbf{92.92\imp{+0.69}} & \gc \textbf{95.90\imp{+2.23}} & \gc \textbf{84.39\imp{+1.24}} & \gc \textbf{84.68\imp{+2.62}} & \gc \textbf{88.23\imp{+3.69}} & \gc \textbf{85.45\imp{+3.24}} \\

&LP++$_\textbf{{(FL)}}$ & 97.07 & 66.51 & 87.23 & 51.35 & 92.01 & 91.04 & 96.24 & 85.87 & 83.09 & 84.44 & 83.48 \\
& \gc w/ \ours & \gc \textbf{97.66\imp{+0.59}} & \gc \textbf{75.26\imp{+8.75}} & \gc \textbf{91.18\imp{+3.95}} & \gc \textbf{53.00\imp{+1.65}} & \gc \textbf{92.49\imp{+0.48}} & \gc \textbf{90.81\wor{-0.24}} & \gc \textbf{95.30\wor{-0.94}} & \gc \textbf{84.46\wor{-1.41}} & \gc \textbf{84.78\imp{+1.69}} & \gc \textbf{87.08\imp{+2.63}} & \gc \textbf{85.20\imp{+1.72}} \\

&ProKeR + CLAP & 97.29 & 68.06 & 88.91 & 49.32 & 92.28 & 94.51 & 94.68 & 82.56 & 83.96 & 86.91 & 83.85 \\
& \gc w/ \ours & \gc \textbf{97.55\imp{+0.26}} & \gc \textbf{75.77\imp{+7.71}} & \gc \textbf{92.16\imp{+3.25}} & \gc \textbf{54.34\imp{+5.02}} & \gc \textbf{91.64\wor{-0.64}} & \gc \textbf{95.98\imp{+1.47}} & \gc \textbf{95.94\imp{+1.25}} & \gc \textbf{83.22\imp{+0.66}} & \gc \textbf{83.38\wor{-0.58}} & \gc \textbf{88.13\imp{+1.22}} & \gc \textbf{85.81\imp{+1.96}} \\

&ProLIP & 97.79 & 70.77 & 88.84 & 56.98 & 93.38 & 93.81 & 96.38 & 87.89 & 86.07 & 88.58 & 86.05 \\
& \gc w/ \ours & \gc \textbf{97.52\wor{-0.26}} & \gc \textbf{74.14\imp{+3.37}} & \gc \textbf{87.91\wor{-0.94}} & \gc \textbf{57.23\imp{+0.24}} & \gc \textbf{91.32\wor{-2.06}} & \gc \textbf{96.07\imp{+2.26}} & \gc \textbf{96.23\wor{-0.15}} & \gc \textbf{86.25\wor{-1.64}} & \gc \textbf{83.98\wor{-2.10}} & \gc \textbf{88.59\imp{+0.01}} & \gc \textbf{85.92\wor{-0.13}} \\

\multicolumn{13}{@{}l@{}}{\dashfill} \\

\multirow{8}{*}{\rotatebox{90}{\textbf{Training Free}}}
&Tip-Adapter & 96.53 & 59.42 & 76.16 & 38.54 & 89.55 & 54.28 & 86.68 & 76.93 & 80.96 & 82.39 & 74.14 \\
& \gc w/ \ours & \gc \textbf{97.74\imp{+1.21}} & \gc \textbf{78.14\imp{+18.72}} & \gc \textbf{92.33\imp{+16.17}} & \gc \textbf{54.49\imp{+15.95}} & \gc \textbf{91.50\imp{+1.95}} & \gc \textbf{93.96\imp{+39.68}} & \gc \textbf{96.94\imp{+10.26}} & \gc \textbf{86.64\imp{+9.71}} & \gc \textbf{84.84\imp{+3.88}} & \gc \textbf{88.65\imp{+6.26}} & \gc \textbf{86.52\imp{+12.38}} \\

&APE & 97.89 & 68.16 & 87.09 & 48.46 & 94.16 & 82.97 & 96.56 & 84.70 & 83.95 & 85.22 & 82.91 \\
& \gc w/ \ours & \gc \textbf{97.99\imp{+0.10}} & \gc \textbf{77.52\imp{+9.36}} & \gc \textbf{90.61\imp{+3.52}} & \gc \textbf{55.03\imp{+6.58}} & \gc \textbf{92.11\wor{-2.05}} & \gc \textbf{94.58\imp{+11.61}} & \gc \textbf{96.98\imp{+0.42}} & \gc \textbf{85.34\imp{+0.65}} & \gc \textbf{84.32\imp{+0.37}} & \gc \textbf{88.28\imp{+3.06}} & \gc \textbf{86.28\imp{+3.36}} \\

&GDA & 93.63 & 62.28 & 87.78 & 45.50 & 77.19 & 86.02 & 82.54 & 70.23 & 73.77 & 80.41 & 75.93 \\
& \gc w/ \ours & \gc \textbf{94.76\imp{+1.12}} & \gc \textbf{67.05\imp{+4.77}} & \gc \textbf{87.57\wor{-0.21}} & \gc \textbf{48.73\imp{+3.23}} & \gc \textbf{80.67\imp{+3.48}} & \gc \textbf{88.19\imp{+2.17}} & \gc \textbf{89.86\imp{+7.32}} & \gc \textbf{72.93\imp{+2.70}} & \gc \textbf{74.98\imp{+1.21}} & \gc \textbf{81.52\imp{+1.11}} & \gc \textbf{78.63\imp{+2.69}} \\

&ProKeR & 97.48 & 71.21 & 86.75 & 52.83 & 93.05 & 93.12 & 94.16 & 85.29 & 85.07 & 87.42 & 84.64 \\
& \gc w/ \ours & \gc \textbf{97.77\imp{+0.29}} & \gc \textbf{77.40\imp{+6.19}} & \gc \textbf{90.49\imp{+3.74}} & \gc \textbf{56.19\imp{+3.36}} & \gc \textbf{91.59\wor{-1.46}} & \gc \textbf{95.33\imp{+2.21}} & \gc \textbf{95.93\imp{+1.76}} & \gc \textbf{85.32\imp{+0.03}} & \gc \textbf{84.22\wor{-0.85}} & \gc \textbf{88.81\imp{+1.39}} & \gc \textbf{86.30\imp{+1.66}} \\

\bottomrule
\end{tabular}%
}}
\caption{\textbf{Severe imbalance with Low effective classes.} We report Top-1 accuracy across 10 datasets.Differences with base methods are shown as subscripts; Positive gains are highlighted in \color{Better}{green}.}
\label{tab:severe_low_vit16_detailed}
\end{table*}

%% file: tables_supp/bal_high_vit16_detailed.tex
\begin{table*}[!htbp]
\centering
\small
\setlength{\tabcolsep}{1pt}
\renewcommand{\arraystretch}{1.08}
\resizebox{0.9\textwidth}{!}{%
\scalebox{1}[0.95]{%
\begin{tabular}{ll|cccccccccc|c}
\toprule
&& \textbf{Caltech} & \textbf{DTD} & \textbf{EuroSAT} & \textbf{FGVC} & \textbf{Food101} & \textbf{Flowers} & \textbf{Pets} & \textbf{Cars} & \textbf{Sun397} & \textbf{UCF101} & \textbf{Mean} \\
\midrule

\multicolumn{13}{c}{\textbf{16-shot}} \\
\midrule

\multirow{10}{*}{\rotatebox{90}{\shortstack{\textbf{Training Based}}}}
&CoOp & 90.24 & 61.09 & 75.17 & 38.24 & 76.44 & 91.61 & 84.57 & 70.29 & 64.45 & 72.76 & 72.49 \\
& \gc w/ \ours & \gc \textbf{92.05\imp{+1.81}} & \gc \textbf{65.13\imp{+4.04}} & \gc \textbf{75.65\imp{+0.48}} & \gc \textbf{37.11\wor{-1.13}} & \gc \textbf{81.99\imp{+5.55}} & \gc \textbf{93.01\imp{+1.40}} & \gc \textbf{89.53\imp{+4.96}} & \gc \textbf{72.33\imp{+2.05}} & \gc \textbf{69.33\imp{+4.87}} & \gc \textbf{76.56\imp{+3.81}} & \gc \textbf{75.27\imp{+2.78}} \\

&LP++$_\textbf{{(CE)}}$ & 93.35 & 63.54 & 77.63 & 39.37 & 84.14 & 90.39 & 89.31 & 73.62 & 70.55 & 78.73 & 76.06 \\
& \gc w/ \ours & \gc \textbf{94.25\imp{+0.90}} & \gc \textbf{67.97\imp{+4.43}} & \gc \textbf{83.51\imp{+5.88}} & \gc \textbf{36.00\wor{-3.36}} & \gc \textbf{84.93\imp{+0.78}} & \gc \textbf{84.87\wor{-5.52}} & \gc \textbf{89.78\imp{+0.47}} & \gc \textbf{70.36\wor{-3.26}} & \gc \textbf{72.25\imp{+1.70}} & \gc \textbf{80.25\imp{+1.52}} & \gc \textbf{76.42\imp{+0.35}} \\

&LP++$_\textbf{{(FL)}}$ & 93.20 & 63.51 & 78.09 & 39.90 & 84.33 & 90.44 & 90.71 & 74.84 & 71.30 & 78.82 & 76.51 \\
& \gc w/ \ours & \gc \textbf{93.58\imp{+0.38}} & \gc \textbf{67.04\imp{+3.54}} & \gc \textbf{81.76\imp{+3.67}} & \gc \textbf{35.22\wor{-4.68}} & \gc \textbf{84.47\imp{+0.15}} & \gc \textbf{82.53\wor{-7.91}} & \gc \textbf{88.92\wor{-1.79}} & \gc \textbf{69.09\wor{-5.75}} & \gc \textbf{72.08\imp{+0.78}} & \gc \textbf{79.33\imp{+0.52}} & \gc \textbf{75.40\wor{-1.11}} \\

&ProKeR + CLAP & 93.77 & 66.23 & 79.98 & 43.22 & 84.96 & 94.27 & 87.61 & 75.27 & 72.68 & 79.66 & 77.76 \\
& \gc w/ \ours & \gc \textbf{94.46\imp{+0.69}} & \gc \textbf{69.47\imp{+3.24}} & \gc \textbf{87.04\imp{+7.06}} & \gc \textbf{41.27\wor{-1.95}} & \gc \textbf{85.54\imp{+0.58}} & \gc \textbf{93.86\wor{-0.40}} & \gc \textbf{92.72\imp{+5.11}} & \gc \textbf{76.32\imp{+1.05}} & \gc \textbf{73.62\imp{+0.94}} & \gc \textbf{80.39\imp{+0.73}} & \gc \textbf{79.47\imp{+1.70}} \\

&ProLIP & 94.18 & 67.38 & 81.65 & 43.20 & 86.36 & 93.41 & 91.40 & 77.74 & 74.61 & 81.93 & 79.19 \\
& \gc w/ \ours & \gc \textbf{94.49\imp{+0.31}} & \gc \textbf{70.20\imp{+2.82}} & \gc \textbf{87.32\imp{+5.67}} & \gc \textbf{43.72\imp{+0.52}} & \gc \textbf{85.07\wor{-1.30}} & \gc \textbf{95.46\imp{+2.05}} & \gc \textbf{92.91\imp{+1.51}} & \gc \textbf{77.61\wor{-0.13}} & \gc \textbf{74.04\wor{-0.57}} & \gc \textbf{82.29\imp{+0.36}} & \gc \textbf{80.31\imp{+1.12}} \\

\multicolumn{13}{@{}l@{}}{\dashfill} \\

\multirow{8}{*}{\rotatebox{90}{\textbf{Training Free}}}
&Tip-Adapter & 86.43 & 33.16 & 33.32 & 21.19 & 53.77 & 34.60 & 37.75 & 44.95 & 56.69 & 57.91 & 45.98 \\
& \gc w/ \ours & \gc \textbf{93.98\imp{+7.54}} & \gc \textbf{64.85\imp{+31.69}} & \gc \textbf{78.69\imp{+45.37}} & \gc \textbf{32.75\imp{+11.56}} & \gc \textbf{85.51\imp{+31.74}} & \gc \textbf{85.87\imp{+51.27}} & \gc \textbf{92.13\imp{+54.38}} & \gc \textbf{72.60\imp{+27.65}} & \gc \textbf{73.30\imp{+16.60}} & \gc \textbf{78.00\imp{+20.09}} & \gc \textbf{75.77\imp{+29.79}} \\

&APE & 91.79 & 55.39 & 50.53 & 24.39 & 77.85 & 61.31 & 79.26 & 52.98 & 66.35 & 72.35 & 63.22 \\
& \gc w/ \ours & \gc \textbf{93.58\imp{+1.78}} & \gc \textbf{65.61\imp{+10.22}} & \gc \textbf{83.42\imp{+32.89}} & \gc \textbf{31.14\imp{+6.76}} & \gc \textbf{83.85\imp{+6.00}} & \gc \textbf{79.72\imp{+18.41}} & \gc \textbf{90.91\imp{+11.65}} & \gc \textbf{64.26\imp{+11.28}} & \gc \textbf{71.66\imp{+5.31}} & \gc \textbf{77.96\imp{+5.60}} & \gc \textbf{74.21\imp{+10.99}} \\

&GDA & 91.60 & 53.66 & 67.12 & 29.29 & 69.66 & 77.28 & 76.86 & 65.19 & 69.82 & 73.33 & 67.38 \\
& \gc w/ \ours & \gc \textbf{91.85\imp{+0.25}} & \gc \textbf{62.52\imp{+8.86}} & \gc \textbf{65.98\wor{-1.14}} & \gc \textbf{34.49\imp{+5.20}} & \gc \textbf{78.42\imp{+8.77}} & \gc \textbf{89.53\imp{+12.24}} & \gc \textbf{83.47\imp{+6.61}} & \gc \textbf{69.20\imp{+4.01}} & \gc \textbf{72.81\imp{+2.99}} & \gc \textbf{75.15\imp{+1.83}} & \gc \textbf{72.34\imp{+4.96}} \\

&ProKeR & 93.84 & 67.33 & 81.40 & 43.17 & 86.36 & 93.34 & 89.56 & 77.45 & 73.70 & 80.39 & 78.65 \\
& \gc w/ \ours & \gc \textbf{94.50\imp{+0.66}} & \gc \textbf{68.37\imp{+1.04}} & \gc \textbf{86.44\imp{+5.04}} & \gc \textbf{39.37\wor{-3.80}} & \gc \textbf{85.66\wor{-0.71}} & \gc \textbf{92.49\wor{-0.85}} & \gc \textbf{92.88\imp{+3.32}} & \gc \textbf{76.92\wor{-0.53}} & \gc \textbf{73.70\imp{+0.00}} & \gc \textbf{80.30\wor{-0.09}} & \gc \textbf{79.06\imp{+0.41}} \\

\midrule
\multicolumn{13}{c}{\textbf{4-shot}} \\
\midrule

\multirow{10}{*}{\rotatebox{90}{\shortstack{\textbf{Training Based}}}}
&CoOp & 86.58 & 50.42 & 64.13 & 28.46 & 72.41 & 85.79 & 82.17 & 60.50 & 56.43 & 65.86 & 65.28 \\
& \gc w/ \ours & \gc \textbf{89.50\imp{+2.92}} & \gc \textbf{57.73\imp{+7.31}} & \gc \textbf{67.95\imp{+3.82}} & \gc \textbf{29.79\imp{+1.33}} & \gc \textbf{76.29\imp{+3.88}} & \gc \textbf{88.22\imp{+2.42}} & \gc \textbf{85.22\imp{+3.04}} & \gc \textbf{63.53\imp{+3.03}} & \gc \textbf{61.42\imp{+5.00}} & \gc \textbf{71.10\imp{+5.23}} & \gc \textbf{69.07\imp{+3.80}} \\

&LP++$_\textbf{{(CE)}}$ & 92.56 & 57.20 & 70.24 & 33.95 & 84.50 & 89.39 & 88.72 & 70.96 & 69.60 & 76.03 & 73.32 \\
& \gc w/ \ours & \gc \textbf{93.20\imp{+0.64}} & \gc \textbf{61.83\imp{+4.62}} & \gc \textbf{75.05\imp{+4.81}} & \gc \textbf{32.72\wor{-1.23}} & \gc \textbf{84.42\wor{-0.08}} & \gc \textbf{83.73\wor{-5.66}} & \gc \textbf{89.84\imp{+1.11}} & \gc \textbf{67.99\wor{-2.97}} & \gc \textbf{70.43\imp{+0.83}} & \gc \textbf{78.01\imp{+1.98}} & \gc \textbf{73.72\imp{+0.41}} \\

&LP++$_\textbf{{(FL)}}$ & 92.69 & 58.12 & 71.47 & 34.10 & 84.46 & 87.92 & 90.19 & 71.47 & 69.89 & 76.60 & 73.69 \\
& \gc w/ \ours & \gc \textbf{92.91\imp{+0.22}} & \gc \textbf{61.51\imp{+3.39}} & \gc \textbf{75.25\imp{+3.78}} & \gc \textbf{31.98\wor{-2.12}} & \gc \textbf{84.09\wor{-0.37}} & \gc \textbf{80.45\wor{-7.47}} & \gc \textbf{88.18\wor{-2.01}} & \gc \textbf{66.71\wor{-4.76}} & \gc \textbf{70.13\imp{+0.25}} & \gc \textbf{77.19\imp{+0.59}} & \gc \textbf{72.84\wor{-0.85}} \\

&ProKeR + CLAP & 93.02 & 61.76 & 73.59 & 35.99 & 85.19 & 91.35 & 88.05 & 71.11 & 70.55 & 77.95 & 74.86 \\
& \gc w/ \ours & \gc \textbf{93.33\imp{+0.31}} & \gc \textbf{64.85\imp{+3.08}} & \gc \textbf{82.05\imp{+8.45}} & \gc \textbf{35.88\wor{-0.11}} & \gc \textbf{83.22\wor{-1.97}} & \gc \textbf{91.38\imp{+0.02}} & \gc \textbf{90.49\imp{+2.44}} & \gc \textbf{70.60\wor{-0.50}} & \gc \textbf{68.91\wor{-1.64}} & \gc \textbf{77.88\wor{-0.07}} & \gc \textbf{75.86\imp{+1.00}} \\

&ProLIP & 93.90 & 62.04 & 78.14 & 36.01 & 86.94 & 89.08 & 91.41 & 74.30 & 73.08 & 79.38 & 76.43 \\
& \gc w/ \ours & \gc \textbf{93.56\wor{-0.34}} & \gc \textbf{64.97\imp{+2.93}} & \gc \textbf{80.80\imp{+2.66}} & \gc \textbf{36.36\imp{+0.35}} & \gc \textbf{83.31\wor{-3.63}} & \gc \textbf{92.47\imp{+3.39}} & \gc \textbf{91.30\wor{-0.11}} & \gc \textbf{71.05\wor{-3.24}} & \gc \textbf{69.94\wor{-3.14}} & \gc \textbf{78.50\wor{-0.88}} & \gc \textbf{76.23\wor{-0.20}} \\

\multicolumn{13}{@{}l@{}}{\dashfill} \\

\multirow{8}{*}{\rotatebox{90}{\textbf{Training Free}}}
&Tip-Adapter & 91.50 & 49.39 & 48.13 & 23.01 & 83.82 & 55.90 & 79.50 & 60.75 & 66.37 & 71.46 & 62.98 \\
& \gc w/ \ours & \gc \textbf{93.33\imp{+1.83}} & \gc \textbf{62.34\imp{+12.95}} & \gc \textbf{82.70\imp{+34.56}} & \gc \textbf{32.13\imp{+9.12}} & \gc \textbf{83.97\imp{+0.14}} & \gc \textbf{85.46\imp{+29.57}} & \gc \textbf{91.80\imp{+12.30}} & \gc \textbf{70.23\imp{+9.49}} & \gc \textbf{70.25\imp{+3.87}} & \gc \textbf{76.83\imp{+5.37}} & \gc \textbf{74.90\imp{+11.92}} \\

&APE & 92.76 & 50.81 & 61.96 & 26.17 & 86.87 & 74.96 & 89.87 & 64.58 & 68.45 & 72.88 & 68.93 \\
& \gc w/ \ours & \gc \textbf{93.38\imp{+0.63}} & \gc \textbf{62.72\imp{+11.91}} & \gc \textbf{81.46\imp{+19.49}} & \gc \textbf{32.43\imp{+6.26}} & \gc \textbf{83.93\wor{-2.94}} & \gc \textbf{86.23\imp{+11.28}} & \gc \textbf{91.44\imp{+1.58}} & \gc \textbf{68.26\imp{+3.67}} & \gc \textbf{70.02\imp{+1.57}} & \gc \textbf{77.08\imp{+4.20}} & \gc \textbf{74.69\imp{+5.76}} \\

&GDA & 88.31 & 50.60 & 63.92 & 26.45 & 61.09 & 74.04 & 68.88 & 52.67 & 60.25 & 67.80 & 61.40 \\
& \gc w/ \ours & \gc \textbf{88.49\imp{+0.18}} & \gc \textbf{55.34\imp{+4.74}} & \gc \textbf{62.20\wor{-1.71}} & \gc \textbf{29.94\imp{+3.49}} & \gc \textbf{68.39\imp{+7.30}} & \gc \textbf{83.04\imp{+9.00}} & \gc \textbf{76.83\imp{+7.95}} & \gc \textbf{56.89\imp{+4.22}} & \gc \textbf{62.56\imp{+2.31}} & \gc \textbf{69.03\imp{+1.24}} & \gc \textbf{65.27\imp{+3.87}} \\

&ProKeR & 93.10 & 61.67 & 76.62 & 36.12 & 86.24 & 89.75 & 88.65 & 72.70 & 71.18 & 77.57 & 75.36 \\
& \gc w/ \ours & \gc \textbf{93.46\imp{+0.36}} & \gc \textbf{63.50\imp{+1.82}} & \gc \textbf{84.55\imp{+7.93}} & \gc \textbf{35.36\wor{-0.76}} & \gc \textbf{83.49\wor{-2.76}} & \gc \textbf{89.79\imp{+0.04}} & \gc \textbf{90.55\imp{+1.89}} & \gc \textbf{71.08\wor{-1.62}} & \gc \textbf{69.98\wor{-1.20}} & \gc \textbf{77.64\imp{+0.07}} & \gc \textbf{75.94\imp{+0.58}} \\

\bottomrule
\end{tabular}%
}}
\caption{\textbf{Near balanced with high effective classes.} We report Top-1 accuracy across 10 datasets. Differences with base methods are
shown as subscripts; Positive gains are highlighted in \color{Better}{green}.}
\label{tab:bal_high_vit16_detailed}
\end{table*}

%% file: tables_supp/bal_low_vit16_detailed.tex
\begin{table*}[!htbp]
\centering
\small
\setlength{\tabcolsep}{1pt}
\renewcommand{\arraystretch}{1.08}
\resizebox{0.9\textwidth}{!}{%
\scalebox{1}[0.95]{%
\begin{tabular}{ll|cccccccccc|c}
\toprule
&& \textbf{Caltech} & \textbf{DTD} & \textbf{EuroSAT} & \textbf{FGVC} & \textbf{Food101} & \textbf{Flowers} & \textbf{Pets} & \textbf{Cars} & \textbf{Sun397} & \textbf{UCF101} & \textbf{Mean} \\
\midrule

\multicolumn{13}{c}{\textbf{16-shot}} \\
\midrule

\multirow{10}{*}{\rotatebox{90}{\shortstack{\textbf{Training Based}}}}
&CoOp & 95.81 & 74.07 & 89.09 & 58.14 & 88.76 & 95.35 & 93.20 & 85.93 & 79.91 & 84.63 & 84.49 \\
& \gc w/ \ours & \gc \textbf{97.52\imp{+1.71}} & \gc \textbf{80.81\imp{+6.74}} & \gc \textbf{90.72\imp{+1.63}} & \gc \textbf{61.34\imp{+3.20}} & \gc \textbf{92.19\imp{+3.43}} & \gc \textbf{97.18\imp{+1.83}} & \gc \textbf{96.73\imp{+3.53}} & \gc \textbf{89.44\imp{+3.51}} & \gc \textbf{84.76\imp{+4.85}} & \gc \textbf{87.81\imp{+3.18}} & \gc \textbf{87.85\imp{+3.36}} \\

&LP++$_\textbf{{(CE)}}$ & 97.81 & 78.60 & 93.72 & 61.06 & 92.56 & 96.65 & 96.15 & 89.11 & 86.53 & 89.53 & 88.17 \\
& \gc w/ \ours & \gc \textbf{98.10\imp{+0.29}} & \gc \textbf{84.34\imp{+5.74}} & \gc \textbf{96.65\imp{+2.93}} & \gc \textbf{62.00\imp{+0.94}} & \gc \textbf{93.33\imp{+0.77}} & \gc \textbf{96.08\wor{-0.57}} & \gc \textbf{97.03\imp{+0.88}} & \gc \textbf{88.47\wor{-0.64}} & \gc \textbf{87.64\imp{+1.11}} & \gc \textbf{91.65\imp{+2.12}} & \gc \textbf{89.53\imp{+1.36}} \\

&LP++$_\textbf{{(FL)}}$ & 97.69 & 79.20 & 93.41 & 61.17 & 92.63 & 95.86 & 96.97 & 89.72 & 86.63 & 89.87 & 88.31 \\
& \gc w/ \ours & \gc \textbf{97.84\imp{+0.15}} & \gc \textbf{83.90\imp{+4.69}} & \gc \textbf{96.96\imp{+3.55}} & \gc \textbf{59.37\wor{-1.79}} & \gc \textbf{93.04\imp{+0.41}} & \gc \textbf{92.57\wor{-3.29}} & \gc \textbf{96.71\wor{-0.26}} & \gc \textbf{87.31\wor{-2.41}} & \gc \textbf{86.92\imp{+0.28}} & \gc \textbf{90.90\imp{+1.04}} & \gc \textbf{88.55\imp{+0.24}} \\

&ProKeR + CLAP & 97.89 & 80.32 & 94.25 & 64.05 & 92.87 & 98.06 & 96.14 & 90.32 & 87.58 & 90.83 & 89.23 \\
& \gc w/ \ours & \gc \textbf{98.18\imp{+0.30}} & \gc \textbf{85.08\imp{+4.76}} & \gc \textbf{96.91\imp{+2.66}} & \gc \textbf{66.87\imp{+2.82}} & \gc \textbf{93.73\imp{+0.87}} & \gc \textbf{98.25\imp{+0.19}} & \gc \textbf{97.72\imp{+1.58}} & \gc \textbf{91.70\imp{+1.38}} & \gc \textbf{88.44\imp{+0.86}} & \gc \textbf{91.68\imp{+0.85}} & \gc \textbf{90.86\imp{+1.63}} \\

&ProLIP & 97.95 & 83.77 & 94.78 & 64.74 & 93.69 & 97.38 & 96.84 & 91.14 & 88.71 & 90.66 & 89.97 \\
& \gc w/ \ours & \gc \textbf{98.04\imp{+0.09}} & \gc \textbf{86.21\imp{+2.44}} & \gc \textbf{95.94\imp{+1.16}} & \gc \textbf{65.65\imp{+0.91}} & \gc \textbf{93.59\wor{-0.09}} & \gc \textbf{98.33\imp{+0.95}} & \gc \textbf{97.31\imp{+0.46}} & \gc \textbf{91.38\imp{+0.24}} & \gc \textbf{88.75\imp{+0.04}} & \gc \textbf{91.67\imp{+1.01}} & \gc \textbf{90.69\imp{+0.72}} \\

\multicolumn{13}{@{}l@{}}{\dashfill} \\

\multirow{8}{*}{\rotatebox{90}{\textbf{Training Free}}}
&Tip-Adapter & 95.64 & 52.62 & 69.89 & 39.01 & 78.38 & 54.46 & 67.76 & 69.86 & 78.01 & 77.54 & 68.32 \\
& \gc w/ \ours & \gc \textbf{98.16\imp{+2.52}} & \gc \textbf{81.82\imp{+29.20}} & \gc \textbf{96.38\imp{+26.48}} & \gc \textbf{59.18\imp{+20.17}} & \gc \textbf{93.00\imp{+14.62}} & \gc \textbf{93.64\imp{+39.18}} & \gc \textbf{97.37\imp{+29.61}} & \gc \textbf{89.80\imp{+19.94}} & \gc \textbf{88.21\imp{+10.19}} & \gc \textbf{90.60\imp{+13.06}} & \gc \textbf{88.81\imp{+20.50}} \\

&APE & 97.80 & 72.76 & 87.23 & 48.52 & 93.47 & 84.48 & 95.54 & 81.81 & 85.33 & 86.72 & 83.37 \\
& \gc w/ \ours & \gc \textbf{98.25\imp{+0.45}} & \gc \textbf{82.64\imp{+9.89}} & \gc \textbf{97.31\imp{+10.07}} & \gc \textbf{61.06\imp{+12.54}} & \gc \textbf{93.67\imp{+0.20}} & \gc \textbf{96.32\imp{+11.83}} & \gc \textbf{97.75\imp{+2.21}} & \gc \textbf{88.75\imp{+6.95}} & \gc \textbf{88.23\imp{+2.89}} & \gc \textbf{90.61\imp{+3.89}} & \gc \textbf{89.46\imp{+6.09}} \\

&GDA & 96.54 & 74.33 & 91.52 & 54.91 & 82.89 & 90.80 & 92.47 & 84.13 & 84.41 & 86.17 & 83.82 \\
& \gc w/ \ours & \gc \textbf{96.74\imp{+0.20}} & \gc \textbf{78.63\imp{+4.29}} & \gc \textbf{92.84\imp{+1.32}} & \gc \textbf{56.71\imp{+1.81}} & \gc \textbf{86.80\imp{+3.91}} & \gc \textbf{94.07\imp{+3.27}} & \gc \textbf{94.82\imp{+2.35}} & \gc \textbf{84.47\imp{+0.34}} & \gc \textbf{85.79\imp{+1.38}} & \gc \textbf{87.40\imp{+1.24}} & \gc \textbf{85.83\imp{+2.01}} \\

&ProKeR & 98.16 & 81.49 & 94.37 & 63.41 & 93.79 & 97.02 & 96.03 & 90.51 & 88.01 & 90.28 & 89.31 \\
& \gc w/ \ours & \gc \textbf{98.43\imp{+0.27}} & \gc \textbf{84.44\imp{+2.95}} & \gc \textbf{96.99\imp{+2.63}} & \gc \textbf{64.98\imp{+1.57}} & \gc \textbf{93.67\wor{-0.12}} & \gc \textbf{97.49\imp{+0.47}} & \gc \textbf{97.48\imp{+1.45}} & \gc \textbf{91.54\imp{+1.02}} & \gc \textbf{88.72\imp{+0.70}} & \gc \textbf{91.18\imp{+0.91}} & \gc \textbf{90.49\imp{+1.18}} \\

\midrule

\multicolumn{13}{c}{\textbf{4-shot}} \\
\midrule

\multirow{10}{*}{\rotatebox{90}{\shortstack{\textbf{Training Based}}}}
&CoOp & 94.75 & 68.38 & 88.72 & 47.00 & 85.33 & 92.81 & 93.65 & 79.23 & 73.20 & 78.78 & 80.18 \\
& \gc w/ \ours & \gc \textbf{96.54\imp{+1.79}} & \gc \textbf{76.24\imp{+7.86}} & \gc \textbf{91.84\imp{+3.12}} & \gc \textbf{51.25\imp{+4.25}} & \gc \textbf{88.64\imp{+3.31}} & \gc \textbf{95.81\imp{+3.00}} & \gc \textbf{95.77\imp{+2.12}} & \gc \textbf{82.78\imp{+3.55}} & \gc \textbf{79.18\imp{+5.98}} & \gc \textbf{84.28\imp{+5.50}} & \gc \textbf{84.23\imp{+4.05}} \\

&LP++$_\textbf{{(CE)}}$ & 97.32 & 72.84 & 90.22 & 54.59 & 92.63 & 95.31 & 95.83 & 87.12 & 85.58 & 87.64 & 85.91 \\
& \gc w/ \ours & \gc \textbf{97.57\imp{+0.25}} & \gc \textbf{79.04\imp{+6.20}} & \gc \textbf{93.17\imp{+2.95}} & \gc \textbf{55.88\imp{+1.29}} & \gc \textbf{92.55\wor{-0.08}} & \gc \textbf{94.81\wor{-0.50}} & \gc \textbf{96.60\imp{+0.77}} & \gc \textbf{86.97\wor{-0.14}} & \gc \textbf{86.54\imp{+0.96}} & \gc \textbf{89.80\imp{+2.16}} & \gc \textbf{87.29\imp{+1.39}} \\

&LP++$_\textbf{{(FL)}}$ & 97.41 & 74.52 & 90.60 & 56.79 & 92.90 & 94.09 & 96.37 & 87.85 & 85.91 & 87.74 & 86.42 \\
& \gc w/ \ours & \gc \textbf{97.72\imp{+0.31}} & \gc \textbf{79.10\imp{+4.58}} & \gc \textbf{93.79\imp{+3.19}} & \gc \textbf{54.55\wor{-2.24}} & \gc \textbf{92.62\wor{-0.28}} & \gc \textbf{92.07\wor{-2.02}} & \gc \textbf{96.09\wor{-0.28}} & \gc \textbf{85.72\wor{-2.13}} & \gc \textbf{86.09\imp{+0.17}} & \gc \textbf{88.68\imp{+0.95}} & \gc \textbf{86.64\imp{+0.22}} \\

&ProKeR + CLAP & 97.95 & 76.79 & 91.09 & 57.64 & 93.36 & 96.18 & 95.98 & 87.10 & 86.62 & 89.82 & 87.25 \\
& \gc w/ \ours & \gc \textbf{97.98\imp{+0.03}} & \gc \textbf{80.37\imp{+3.58}} & \gc \textbf{93.41\imp{+2.31}} & \gc \textbf{59.50\imp{+1.86}} & \gc \textbf{92.27\wor{-1.09}} & \gc \textbf{96.87\imp{+0.70}} & \gc \textbf{96.52\imp{+0.54}} & \gc \textbf{87.21\imp{+0.10}} & \gc \textbf{86.08\wor{-0.54}} & \gc \textbf{90.47\imp{+0.65}} & \gc \textbf{88.07\imp{+0.81}} \\

&ProLIP & 98.01 & 77.63 & 93.34 & 61.14 & 94.42 & 95.85 & 97.22 & 90.02 & 87.58 & 90.54 & 88.57 \\
& \gc w/ \ours & \gc \textbf{97.91\wor{-0.10}} & \gc \textbf{80.53\imp{+2.90}} & \gc \textbf{91.91\wor{-1.44}} & \gc \textbf{60.57\wor{-0.57}} & \gc \textbf{92.46\wor{-1.95}} & \gc \textbf{97.07\imp{+1.23}} & \gc \textbf{97.07\wor{-0.15}} & \gc \textbf{89.04\wor{-0.98}} & \gc \textbf{86.01\wor{-1.57}} & \gc \textbf{90.75\imp{+0.22}} & \gc \textbf{88.33\wor{-0.24}} \\

\multicolumn{13}{@{}l@{}}{\dashfill} \\

\multirow{8}{*}{\rotatebox{90}{\textbf{Training Free}}}
&Tip-Adapter & 97.36 & 68.94 & 82.45 & 48.52 & 93.40 & 77.95 & 93.73 & 83.80 & 84.47 & 86.20 & 81.68 \\
& \gc w/ \ours & \gc \textbf{98.07\imp{+0.71}} & \gc \textbf{80.46\imp{+11.52}} & \gc \textbf{94.48\imp{+12.03}} & \gc \textbf{57.37\imp{+8.85}} & \gc \textbf{92.73\wor{-0.66}} & \gc \textbf{95.07\imp{+17.11}} & \gc \textbf{97.36\imp{+3.64}} & \gc \textbf{88.33\imp{+4.53}} & \gc \textbf{86.64\imp{+2.17}} & \gc \textbf{89.70\imp{+3.50}} & \gc \textbf{88.02\imp{+6.34}} \\

&APE & 97.99 & 69.38 & 89.63 & 53.13 & 94.52 & 87.71 & 96.99 & 86.88 & 85.07 & 85.85 & 84.71 \\
& \gc w/ \ours & \gc \textbf{98.13\imp{+0.14}} & \gc \textbf{80.29\imp{+10.91}} & \gc \textbf{94.67\imp{+5.04}} & \gc \textbf{57.42\imp{+4.29}} & \gc \textbf{92.31\wor{-2.21}} & \gc \textbf{95.89\imp{+8.18}} & \gc \textbf{97.10\imp{+0.11}} & \gc \textbf{87.80\imp{+0.92}} & \gc \textbf{86.71\imp{+1.65}} & \gc \textbf{89.93\imp{+4.08}} & \gc \textbf{88.03\imp{+3.31}} \\

&GDA & 95.25 & 69.06 & 88.90 & 49.67 & 80.97 & 89.94 & 86.01 & 76.49 & 78.58 & 83.08 & 79.79 \\
& \gc w/ \ours & \gc \textbf{95.83\imp{+0.57}} & \gc \textbf{73.91\imp{+4.86}} & \gc \textbf{89.16\imp{+0.26}} & \gc \textbf{51.51\imp{+1.84}} & \gc \textbf{83.55\imp{+2.58}} & \gc \textbf{91.15\imp{+1.21}} & \gc \textbf{92.23\imp{+6.22}} & \gc \textbf{78.31\imp{+1.81}} & \gc \textbf{79.54\imp{+0.96}} & \gc \textbf{84.37\imp{+1.29}} & \gc \textbf{81.95\imp{+2.16}} \\

&ProKeR & 97.87 & 77.66 & 92.12 & 57.14 & 93.82 & 95.29 & 96.76 & 87.91 & 86.65 & 89.17 & 87.44 \\
& \gc w/ \ours & \gc \textbf{98.11\imp{+0.24}} & \gc \textbf{81.37\imp{+3.71}} & \gc \textbf{94.97\imp{+2.85}} & \gc \textbf{58.87\imp{+1.73}} & \gc \textbf{92.26\wor{-1.56}} & \gc \textbf{96.75\imp{+1.47}} & \gc \textbf{97.27\imp{+0.51}} & \gc \textbf{88.08\imp{+0.17}} & \gc \textbf{86.19\wor{-0.46}} & \gc \textbf{89.92\imp{+0.75}} & \gc \textbf{88.38\imp{+0.94}} \\

\bottomrule
\end{tabular}%
}}
\caption{\textbf{Near balanced with low effective classes.} We report Top-1 accuracy across 10 datasets. Differences with base methods are
shown as subscripts; Positive gains are highlighted in \color{Better}{green}.}
\label{tab:bal_low_vit16_detailed}
\end{table*}

%% file: main.bib
@String(IJCV = {Int. J. Comput. Vis.})

@String(CVPR= {IEEE Conf. Comput. Vis. Pattern Recog.})

@String(ECCV= {Eur. Conf. Comput. Vis.})

@String(ICLR = {Int. Conf. Learn. Represent.})

@String(IJCV  = {IJCV})

@String(CVPR  = {CVPR})

@String(ECCV  = {ECCV})

@String(ICLR  = {ICLR})

@inproceedings{radford2021clip,
  title        = {{Learning Transferable Visual Models From Natural Language Supervision}},
  author       = {Radford, Alec and Kim, Jong Wook and Hallacy, Chris and Ramesh, Aditya and Goh, Gabriel and Agarwal, Sandhini and Sastry, Girish and Askell, Amanda and Mishkin, Pamela and Clark, Jack and Krueger, Gretchen and Sutskever, Ilya},
  booktitle    = {Proceedings of the 38th International Conference on Machine Learning (ICML)},
  year         = {2021},
  url          = {https://arxiv.org/abs/2103.00020}
}

@inproceedings{jia2021scaling,
  title={Scaling up visual and vision-language representation learning with noisy text supervision},
  author={Jia, Chao and Yang, Yinfei and Xia, Ye and Chen, Yi-Ting and Parekh, Zarana and Pham, Hieu and Le, Quoc and Sung, Yun-Hsuan and Li, Zhen and Duerig, Tom},
  booktitle={International conference on machine learning},
  pages={4904--4916},
  year={2021},
  organization={PMLR}
}

@article{xu2023metaclip,
  title        = {{Demystifying CLIP Data}},
  author       = {Xu, Hu and Xie, Saining and Tan, Xiaoqing Ellen and Huang, Po-Yao and Howes, Russell and Sharma, Vasu and Li, Shang-Wen and Ghosh, Gargi and Zettlemoyer, Luke and Feichtenhofer, Christoph},
  journal      = {arXiv preprint arXiv:2309.16671},
  year         = {2023},
  note         = {Introduces MetaCLIP dataset and scaling insights},
  url          = {https://arxiv.org/abs/2309.16671}
}

@inproceedings{zhou2022cocoop,
  title        = {{Conditional Prompt Learning for Vision-Language Models}},
  author       = {Zhou, Kaiyang and Yang, Jingkang and Loy, Chen Change and Liu, Ziwei},
  booktitle    = {Proceedings of the IEEE/CVF Conference on Computer Vision and Pattern Recognition (CVPR)},
  year         = {2022},
  url          = {https://arxiv.org/abs/2203.05557}
}

@article{zhou2022coop,
  title   = {{Learning to Prompt for Vision-Language Models}},
  author  = {Zhou, Kaiyang and Yang, Jingkang and Loy, Chen Change and Liu, Ziwei},
  journal = {IJCV},
  volume  = {130},
  number  = {9},
  pages   = {2337--2348},
  year    = {2022},
  doi     = {10.1007/s11263-022-01653-1},
  url     = {https://link.springer.com/article/10.1007/s11263-022-01653-1}
}

@inproceedings{bendou2025proker,
  title={Proker: A kernel perspective on few-shot adaptation of large vision-language models},
  author={Bendou, Yassir and Ouasfi, Amine and Gripon, Vincent and Boukhayma, Adnane},
  booktitle={2025 IEEE/CVF Conference on Computer Vision and Pattern Recognition (CVPR)},
  pages={25092--25102},
  year={2025},
  organization={IEEE}
}

@inproceedings{wang2024baseline,
  title={{A Hard-to-Beat Baseline for Training-free CLIP-based Adaptation}},
  author={Wang, Zhengbo and Liang, Jian and Sheng, Lijun and He, Ran and Wang, Zilei and Tan, Tieniu},
  booktitle={The Twelfth International Conference on Learning Representations (ICLR)},
  year={2024}
}

@inproceedings{zhang2022tipadapter,
  title={{Tip-adapter: Training-free adaption of clip for few-shot classification}},
  author={Zhang, Renrui and Zhang, Wei and Fang, Rongyao and Gao, Peng and Li, Kunchang and Dai, Jifeng and Qiao, Yu and Li, Hongsheng},
  booktitle={European conference on computer vision},
  pages={493--510},
  year={2022},
  organization={Springer}
}

@inproceedings{zhu2023not,
  title={{Not All Features Matter: Enhancing Few-shot Clip with Adaptive Prior Refinement}},
  author={Zhu, Xiangyang and Zhang, Renrui and He, Bowei and Zhou, Aojun and Wang, Dong and Zhao, Bin and Gao, Peng},
  booktitle={Proceedings of the IEEE/CVF international conference on computer vision},
  pages={2605--2615},
  year={2023}
}

@inproceedings{huang2024lp++,
  title={Lp++: A surprisingly strong linear probe for few-shot clip},
  author={Huang, Yunshi and Shakeri, Fereshteh and Dolz, Jose and Boudiaf, Malik and Bahig, Houda and Ben Ayed, Ismail},
  booktitle={Proceedings of the IEEE/CVF conference on computer vision and pattern recognition},
  pages={23773--23782},
  year={2024}
}

@inproceedings{clap24,
  title        = {{A Closer Look at the Few-Shot Adaptation of Large Vision-Language Models}},
  author       = {Silva-Rodr\'iguez, Julio and Hajimiri, Sina and Ben Ayed, Ismail and Dolz, Jose},
  booktitle    = {Proceedings of the IEEE/CVF Conference on Computer Vision and Pattern Recognition (CVPR)},
  year         = {2024}
}

@inproceedings{lu2022prompt,
  title={{Prompt Distribution Learning}},
  author={Lu, Yuning and Liu, Jianzhuang and Zhang, Yonggang and Liu, Yajing and Tian, Xinmei},
  booktitle={Proceedings of the IEEE/CVF conference on computer vision and pattern recognition},
  pages={5206--5215},
  year={2022}
}

@inproceedings{lin2023multimodality,
  title={{Multimodality helps unimodality: Cross-modal few-shot learning with multimodal models}},
  author={Lin, Zhiqiu and Yu, Samuel and Kuang, Zhiyi and Pathak, Deepak and Ramanan, Deva},
  booktitle={Proceedings of the IEEE/CVF conference on computer vision and pattern recognition},
  pages={19325--19337},
  year={2023}
}

@article{zanella2024boosting,
  title={{Boosting Vision-Language Models with Transduction}},
  author={Zanella, Maxime and G{\'e}rin, Beno{\^\i}t and Ayed, Ismail},
  journal={Advances in Neural Information Processing Systems},
  volume={37},
  pages={62223--62256},
  year={2024}
}

@article{lange2000optimization,
  title={Optimization transfer using surrogate objective functions},
  author={Lange, Kenneth and Hunter, David R and Yang, Ilsoon},
  journal={Journal of computational and graphical statistics},
  volume={9},
  number={1},
  pages={1--20},
  year={2000},
  publisher={Taylor \& Francis}
}

@inproceedings{hu2022lora,
  title   = {{LoRA: Low-Rank Adaptation of Large Language Models}},
  author    = {Hu, Edward J. and Shen, Yelong and Wallis, Phillip and Allen-Zhu, Zeyuan and Li, Yuanzhi and Wang, Shean and Wang, Lu and Chen, Weizhu},
  booktitle = {International Conference on Learning Representations (ICLR)},
  year      = {2022},
  url       = {https://openreview.net/forum?id=nZeVKeeFYf9}
}

@article{boudiaf2020information,
  title={Information maximization for few-shot learning},
  author={Boudiaf, Malik and Ziko, Imtiaz and Rony, J{\'e}r{\^o}me and Dolz, Jos{\'e} and Piantanida, Pablo and Ben Ayed, Ismail},
  journal={Advances in Neural Information Processing Systems},
  volume={33},
  pages={2445--2457},
  year={2020}
}

@inproceedings{parkhi2012pets,
  title     = {{Cats and Dogs}},
  author    = {Parkhi, Omkar M. and Vedaldi, Andrea and Zisserman, Andrew and Jawahar, C. V.},
  booktitle = {Proceedings of the IEEE Conference on Computer Vision and Pattern Recognition (CVPR)},
  year      = {2012},
  url       = {https://www.robots.ox.ac.uk/~vedaldi/assets/pubs/parkhi12cat.pdf}
}

@inproceedings{krause2013stanfordcars,
  title        = {{3D Object Representations for Fine-Grained Categorization}},
  author       = {Krause, Jonathan and Stark, Michael and Deng, Jia and Fei-Fei, Li},
  booktitle    = {IEEE International Conference on Computer Vision Workshops (ICCVW)},
  year         = {2013}
}

@article{maji2013fgvcaircraft,
  title={{Fine-grained Visual Classification of Aircraft}},
  author={Maji, Subhransu and Rahtu, Esa and Kannala, Juho and Blaschko, Matthew and Vedaldi, Andrea},
  journal={arXiv preprint arXiv:1306.5151},
  year={2013}
}

@inproceedings{cimpoi2014dtd,
  title        = {{Describing Textures in the Wild}},
  author       = {Cimpoi, Mircea and Maji, Subhransu and Kokkinos, Iasonas and Mohamed, Sammy and Vedaldi, Andrea},
  booktitle    = {IEEE Conference on Computer Vision and Pattern Recognition (CVPR)},
  year         = {2014}
}

@inproceedings{bossard2014food101,
  title        = {{Food-101 — Mining Discriminative Components with Random Forests}},
  author       = {Bossard, Lukas and Guillaumin, Matthieu and Van Gool, Luc},
  booktitle    = {European Conference on Computer Vision (ECCV)},
  year         = {2014}
}

@inproceedings{xiao2010sun397,
  title        = {{SUN Database: Large-scale Scene Recognition from Abbey to Zoo}},
  author       = {Xiao, Jianxiong and Hays, James and Ehinger, Krista and Oliva, Aude and Torralba, Antonio},
  booktitle    = {IEEE Conference on Computer Vision and Pattern Recognition (CVPR)},
  year         = {2010}
}

@article{helber2019eurosat,
  title={{EuroSAT: A Novel Dataset and Deep Learning Benchmark for Land Use and Land Cover Classification}},
  author={Helber, Patrick and Bischke, Benjamin and Dengel, Andreas and Borth, Damian},
  journal={IEEE Journal of Selected Topics in Applied Earth Observations and Remote Sensing},
  volume={12},
  number={7},
  pages={2217--2226},
  year={2019},
  publisher={IEEE}
}

@inproceedings{nilsback2008flowers102,
  title        = {{Automated Flower Classification over a Large Number of Classes}},
  author       = {Nilsback, Maria-Elena and Zisserman, Andrew},
  booktitle    = {Proceedings of the Indian Conference on Computer Vision, Graphics and Image Processing (ICVGIP)},
  year         = {2008}
}

@inproceedings{fei2004caltech101,
  title     = {{Learning Generative Visual Models from Few Training Examples: An Incremental Bayesian Approach Tested on 101 Object Categories}},
  author    = {Fei-Fei, Li and Fergus, Rob and Perona, Pietro},
  booktitle = {Proceedings of the IEEE Conference on Computer Vision and Pattern Recognition (CVPR) Workshops},
  year      = {2004}
}

@article{soomro2012ucf101,
  title   = {{UCF101: A Dataset of 101 Human Actions Classes from Videos in the Wild}},
  author  = {Soomro, Khurram and Zamir, Amir Roshan and Shah, Mubarak},
  journal = {arXiv preprint arXiv:1212.0402},
  year    = {2012},
  url     = {https://arxiv.org/abs/1212.0402}
}

@inproceedings{deng2009imagenet,
  title        = {{ImageNet: A Large-Scale Hierarchical Image Database}},
  author       = {Deng, Jia and Dong, Wei and Socher, Richard and Li, Li-Jia and Li, Kai and Fei-Fei, Li},
  booktitle    = {IEEE Conference on Computer Vision and Pattern Recognition (CVPR)},
  year         = {2009}
}

@article{dosovitskiy2020image,
  title={An image is worth 16x16 words: Transformers for image recognition at scale},
  author={Dosovitskiy, Alexey and Beyer, Lucas and Kolesnikov, Alexander and Weissenborn, Dirk and Zhai, Xiaohua and Unterthiner, Thomas and Dehghani, Mostafa and Minderer, Matthias and Heigold, Georg and Gelly, Sylvain and others},
  journal={arXiv preprint arXiv:2010.11929},
  year={2020}
}

@inproceedings{fahes2026clip,
  title={Clip’s visual embedding projector is a few-shot cornucopia},
  author={Fahes, Mohammad and Vu, Tuan-Hung and Bursuc, Andrei and P{\'e}rez, Patrick and De Charette, Raoul},
  booktitle={2026 IEEE/CVF Winter Conference on Applications of Computer Vision (WACV)},
  pages={3254--3264},
  year={2026},
  organization={IEEE}
}
